\documentclass[11pt]{article}
\usepackage[utf8]{inputenc}
\usepackage[english]{babel}
\usepackage[margin=0.85in]{geometry}
\usepackage{amsmath,amssymb,amsfonts,amsthm}
\usepackage{graphicx,booktabs,multirow,array}
\usepackage{hyperref,xcolor,url}
\usepackage{algorithm,algpseudocode}
\usepackage{setspace}
\usepackage{enumitem}
\usepackage{subcaption}
\usepackage{natbib}
\usepackage{appendix}
\usepackage{tabularx}
\usepackage{placeins}
\usepackage{titlesec}
\usepackage{mathtools}

\titlespacing*{\section}{0pt}{1.3ex plus 0.6ex minus .2ex}{0.8ex plus .2ex}
\titlespacing*{\subsection}{0pt}{1.1ex plus 0.5ex minus .2ex}{0.6ex plus .2ex}
\titlespacing*{\subsubsection}{0pt}{0.9ex plus 0.4ex minus .2ex}{0.5ex plus .2ex}
\theoremstyle{definition}
\newtheorem{definition}{Definition}[section]
\newtheorem{theorem}{Theorem}[section]

\newtheorem{proposition}[theorem]{Proposition}
\newtheorem{corollary}[theorem]{Corollary}

\theoremstyle{remark}
\newtheorem{remark}{Remark}[section]
\newtheorem{assumption}{Assumption}[section]

\hypersetup{
    colorlinks=true,
    linkcolor=blue!70!black,
    citecolor=blue!70!black,
    urlcolor=blue!70!black
}

\newcommand{\II}{\mathcal{I}}
\newcommand{\Bmax}{B_{\max}}
\newcommand{\Dtr}{\mathcal{D}^{\mathrm{tr}}}
\newcommand{\Dtst}{\mathcal{D}^{\mathrm{tst}}}
\newcommand{\IG}{\mathrm{IG}}
\newcommand{\EU}{\mathrm{EU}}
\newcommand{\R}{\mathbb{R}}

\newcommand{\VC}{\mathrm{VC}}

\newcommand{\IAA}{IAIML-A}
\newcommand{\IAR}{IAIML-R}

\title{\Large \textbf{Complexity-Budgeted, Interaction-Aware \\Interpretable Model for Tabular Data}}

\author{
	Srikumar Krishnamoorthy\\
	Information Systems Area, Indian Institute of Management Ahmedabad, India
}

\date{}

\begin{document}

\maketitle

\begin{abstract}
Inherently interpretable classifiers for tabular data typically rely on sparse features, rules, or patterns that users can inspect directly. The marginal feature-screening step common to these methods can discard variables whose predictive value emerges only through joint configurations with other variables. We present Interaction Aware Interpretable Machine Learning (IAIML), a framework that addresses this limitation through three coordinated mechanisms: adaptive per-feature discretization, finite-grid pairwise interaction scoring, and a partitioned explanation budget. Detected interactions are routed through one of two strategies: relaxing the screening filter so that interaction-supported variables enter the pattern search, or constructing explicit pair terms for a sparse downstream classifier. On a 40-dataset panel comprising 24 real-world tabular benchmarks and 16 synthetic interaction stress tests, evaluated under nested cross-validation, IAIML achieves mean AUC within 1.4 points of tuned gradient-boosted ensembles while requiring roughly 14--28 times fewer fitted explanation components. On datasets with strong pairwise interaction structure and low marginal signal, IAIML outperforms all baselines. Among compact interpretable methods, IAIML is comparable to RuleFit in AUC and component count and is less expensive to tune. EBM obtains a small but significant AUC advantage across the full panel, with a substantially larger lookup-table footprint. Performance degrades on datasets requiring higher-order interactions beyond the pairwise scope. Component-isolated ablations confirm that adaptive discretization and interaction-aware admission each contribute incrementally. These results support IAIML as a compact, interaction-aware framework appropriate for settings where bounded explanation size and controlled treatment of feature interactions are design requirements.

\vspace{0.5em}
\end{abstract}

\FloatBarrier
\section{Introduction}
\label{sec:intro}

Machine learning is increasingly used to support important decisions in healthcare,
finance, criminal justice, and public policy. In these settings, stakeholders require
that automated decisions be explainable, auditable, and contestable
\citep{Rudin2019, Guidotti2019, DeBock2024}. This requirement has motivated a growing
body of work on \emph{inherently interpretable} models, which provide transparency by
construction, as opposed to post-hoc explanation methods that approximate opaque models
after the fact \citep{Ribeiro2016, Lundberg2017}. Interpretable model families include
decision trees \citep{Breiman1984}, rule lists \citep{Letham2015, Angelino2018},
scoring systems \citep{Ustun2016}, generalized additive models
\citep{Lou2012, Lou2013}, and pattern-based classifiers \citep{Krishnamoorthy2024}.
A common thread across these families is sparsity: the model's prediction is traceable
to a small number of features, rules, or patterns that a human reviewer can inspect
directly.

\par Achieving sparsity, however, typically requires an early screening step that ranks
candidate features by their individual association with the target variable. This
marginal screening is a natural choice for auditability and computational control,
but it introduces a systematic vulnerability. Two variables may each be individually
uninformative while their joint configuration is highly predictive of the outcome.
Consider a clinical setting in which a patient's risk depends on the combination of
two biomarkers, neither of which is predictive on its own. A screening rule that ranks
features by marginal association will discard both biomarkers before their joint signal
can be observed. This problem is not limited to rare edge cases. Feature interactions
are well documented in genomics, clinical risk prediction, and credit scoring
\citep{Dumitrescu2022}, and their suppression can
degrade both accuracy and the fidelity of the resulting explanations.

\par The vulnerability is especially direct in pattern-mining methods that assign item
utilities from marginal association before evaluating joint patterns. In such
methods, a feature with negligible marginal signal receives low utility regardless
of how informative it becomes when combined with a partner. Other interpretable
families handle interactions differently. Tree-derived rules can encode split-based
interactions, but only those that the greedy splitting heuristic happens to explore.
Explainable Boosting Machines (EBM) add selected pairwise terms through
residual-based boosting \citep{Lou2013}, but residual-based pair selection can miss
pairs whose constituent features contribute negligible main effects. RuleFit
\citep{Friedman2008} inherits interaction structure from a tree ensemble, but provides
no mechanism to distinguish rules that capture genuine synergy from those that capture
redundant marginal signal. In each case, interaction detection is either absent,
implicit, or coupled to marginal signal strength.

\par This paper introduces Interaction Aware Interpretable Machine Learning (IAIML), a
pattern-classification framework that detects pairwise synergy independently of
marginal signal strength and channels it into a sparse, budget-constrained classifier.
IAIML builds on the HUG-IML pattern-mining method \citep{Krishnamoorthy2024} and
addresses a specific gap: the suppression of interaction-only
variables by marginal-utility-based candidate admission. The framework scores pairwise
interactions on a coarse discretization grid, then routes detected synergy through one
of two strategies. The first relaxes the screening filter so that interaction-supported
variables can enter the pattern search. The second constructs explicit algebraic pair
terms and supplies them to the downstream classifier. A partitioned explanation budget
caps the total number of auditable components, ensuring that the added representational
capacity does not compromise inspectability. Every prediction remains traceable to a
bounded set of patterns, original features, and, where applicable, named pair terms.

\par This paper makes the following three contributions:

\begin{enumerate}
\item Presents an information-theoretic interaction criterion for marginal-signal-independent synergy detection. The criterion quantifies pairwise excess joint information from discretized contingency tables independently of marginal signal strength. Finite-grid concentration and ranking-consistency guarantees establish when the resulting source selection can be trusted at a given sample size.

\item Introduces IAIML, a three-stage framework that operationalizes the information-theoretic interaction criterion. The framework combines adaptive per-feature discretization, two separable routing strategies for channeling detected synergy into the explanation (\IAR{} for pattern-level admission, \IAA{} for named algebraic pair terms), and a partitioned complexity budget that caps the total number of auditable explanation components.

\item Empirically evaluates the framework under a shared nested cross-validation protocol. On a 40-dataset panel (24~real-world benchmarks, 16~synthetic interaction stress tests), IAIML achieves mean AUC within 1.4 points of tuned gradient-boosted ensembles while requiring roughly 14--28 times fewer fitted explanation components. On a subset of datasets with strong pairwise interaction structure and low marginal signal, IAIML outperforms all baselines. Component-isolated ablations confirm that each design element contributes incrementally.
\end{enumerate}

The remainder of the paper is organized as follows. Section~\ref{sec:related} reviews
related work. Section~\ref{sec:prelim} introduces information-theoretic preliminaries
and the pattern-mining formalism. Section~\ref{sec:methodology} describes the proposed
three-stage framework. Section~\ref{sec:theory} develops the theoretical motivation
and complexity analysis. Section~\ref{sec:experiments} presents the empirical study and shares the key findings.
Section~\ref{sec:discussion} discusses implications and limitations, and
Section~\ref{sec:conclusion} presents concluding remarks.

\FloatBarrier
\section{Related Work}
\label{sec:related}

\emph{Tree-based and rule-based methods.} Decision trees \citep{Breiman1984, Bertsimas2017} provide direct rule-based transparency and can implicitly capture interactions through recursive partitioning. Their interpretability and computational efficiency make them a natural first choice in many applied settings. The interactions that emerge are determined by the greedy splitting heuristic, which favors features with strong marginal signal. Rule list and rule set methods \citep{Letham2015, Angelino2018, Lakkaraju2016} offer compact Boolean decision logic optimized for auditability; CORELS \citep{Angelino2018} provides certifiably optimal rule lists, a guarantee no other method in this family offers, though it requires pre-binarized features. Scoring systems \citep{Ustun2016} produce integer point tables that are directly usable in clinical and regulatory settings. RuleFit \citep{Friedman2008} bridges ensemble learning and interpretability by extracting conjunctive rules from a tree ensemble and fitting a LASSO-regularized linear model over rules and original features. This architecture is effective at capturing interaction structure inherited from tree splits, though the interaction coverage is governed by the tree-growing heuristic rather than a targeted synergy criterion.

\emph{Additive and post-hoc methods.} Generalized additive models provide visualizable per-feature shape functions. Explainable Boosting Machines (EBM) \citep{Lou2012, Lou2013, Nori2019, Chang2021} represent the current state of the art among inherently interpretable methods, fitting pairwise interaction terms through greedy cyclic boosting guided by the FAST algorithm \citep{Lou2013}. FAST selects pairs by residual-error reduction after main effects are fit, a strategy that is effective when interacting features also carry moderate marginal signal. A structural consequence of residual-based selection is that pairs whose constituent features carry negligible marginal signal contribute little residual to explain and may therefore receive lower priority, even when their joint configuration is informative. Neural Additive Models \citep{Agarwal2021} and TabNet \citep{Arik2021} extend the additive paradigm with flexible or attention-based alternatives, trading some global transparency for capacity. Sparse linear models with engineered features \citep{Gosiewska2021, Dumitrescu2022, Liu2024, Chen2024ejor} augment logistic regression with nonlinear transforms, an approach that has shown promise in credit scoring and other regulated domains. Post-hoc explanation methods such as LIME \citep{Ribeiro2016} and SHAP \citep{Lundberg2017} provide local approximations of opaque models; their role is complementary to inherently interpretable models, which build transparency into the model itself \citep{Rudin2019, Borgonovo2024, DeBock2024}.

\emph{Pattern mining and Interaction detection.} High Utility Itemset mining \citep{Wu2015, FournierViger2014, Zida2017, Krishnamoorthy2017, Tseng2016} identifies itemsets whose aggregate utility exceeds a threshold. The HUG-IML framework \citep{Krishnamoorthy2024} adapts this machinery for supervised classification via feature--target correlation utilities and information-gain filtering. IAIML uses HUG-IML-style high-utility gain pattern mining as the pattern-generation backbone, but changes the upstream candidate-admission mechanism and the downstream representation budget. The original HUG-IML formulation assigns feature utilities from marginal feature--target association and filters mined patterns by information gain. IAIML keeps the interpretable pattern vocabulary but adds adaptive discretization, explicit pairwise interaction scoring, and one of two controlled pathways for allowing low-marginal but interaction-supported variables to enter the final sparse classifier. Section~\ref{sec:theory} shows formally why HUG-IML's marginal-utility weighting cannot detect interaction-only patterns.

Feature interaction detection more broadly draws on ANOVA-based tree diagnostics \citep{Hooker2007}, Friedman's H-statistic \citep{Friedman2008}, Shapley interaction indices \citep{Grabisch1999, Sundararajan2020, Borgonovo2024}, and Partial Information Decomposition (PID) \citep{Williams2010, Griffith2014, Bertschinger2014}.
The interaction score used by IAIML is a simplified, grid-based measure of net interaction rather than a full PID estimate. In PID terms, the score compares how much a pair of features tells us about the label when considered jointly versus the sum of what each feature tells us on its own. This comparison blends two different effects together: synergy (information that only emerges when the features are considered together) and redundancy (information that both features already share individually), without separating them. A positive score simply means the pair, taken together, carries more information about the label than the two features do separately. This makes the score useful as a quick screening tool for spotting pairs of features whose joint behavior matters even when neither feature looks informative on its own. 

Discretization of continuous features, meaning the conversion of them into discrete bins or categories, is a classical preprocessing step \citep{Fayyad1993, Dougherty1995, Kerber1992, Tsai2008, Garcia2013}. Most existing methods choose how to bin each feature by optimizing a criterion for that feature alone, without regard to other features. IAIML builds on this same general approach but extends it by also taking feature interactions into account when deciding how each feature should be discretized.

Table~\ref{tab:landscape} summarizes the related literature. The limitation IAIML addresses is most direct for pattern-mining frameworks that use marginal feature--target utilities before evaluating candidate itemsets. Related interpretable families such as EBM and RuleFit can represent pairwise structure, but they obtain it through residual-based boosting or tree-derived rule extraction rather than through an explicit marginal-signal-independent synergy score tied to a bounded pattern vocabulary. IAIML combines such a criterion with an inherently interpretable downstream classifier. Section~\ref{sec:experiments} evaluates it directly against representative baselines: tuned gradient-boosted ensembles as strong predictive baselines, EBM as the leading GAM-with-interactions method, and RuleFit as the leading tree-derived rule-ensemble method.

\begin{table}[t]
\centering
\caption{Summary of interpretable model families with respect to interaction handling.}
\label{tab:landscape}\setlength{\tabcolsep}{3pt}
\small
\begin{tabular}{lcccc}
\toprule
\textbf{Method family} & \textbf{Inherent} & \textbf{Explicit interaction} & \textbf{Marginal-independent} & \textbf{Complexity-} \\
 & \textbf{transparency?} & \textbf{criterion?} & \textbf{detection?} & \textbf{bounded?} \\
\midrule
Trees / rule lists & Yes & No & No & Partially \\
Scoring systems & Yes & No & No & Yes \\
RuleFit & Yes & No (implicit via trees) & No & Yes \\
EBM & Yes & Yes (FAST residual) & No & No \\
Post-hoc (SHAP, LIME) & No & No & N/A & N/A \\
IAIML (this paper) & Yes & Yes & Yes & Yes \\
\bottomrule
\end{tabular}
\end{table}

\FloatBarrier
\section{Preliminaries}
\label{sec:prelim}

Let $\mathcal{D} = \{(\mathbf{x}^{(i)}, y^{(i)})\}_{i=1}^n$ denote a labeled dataset with $p$ predictors $\mathbf{x}^{(i)} \in \R^p$ and binary\footnote{Generalization to multiclass settings is straightforward; we focus on binary classification for clarity of exposition.} labels $y^{(i)} \in \{0,1\}$. We write $f_j$ for the $j$-th feature and $x_j^{(i)}$ for its value in the $i$-th instance. Let $\Dtr$ and $\Dtst$ denote training and test splits obtained by stratified partitioning.

$I(X;Y)$ denotes the mutual information between random variables $X$ and $Y$, estimated from empirical distributions on discretized contingency tables:
\begin{equation}
\label{eq:mi}
\hat{I}(X;Y) = \sum_{x \in \mathcal{X}} \sum_{y \in \mathcal{Y}} \hat{p}(x,y) \log \frac{\hat{p}(x,y)}{\hat{p}(x)\hat{p}(y)},
\end{equation}
where $\hat{p}$ denotes the empirical distribution. We use natural logarithms (nats); results hold for any logarithmic base up to a constant.

The \emph{information gain} of a discretized feature $f_j$ (with $B_j$ bins) with respect to the target is $\IG(f_j) \equiv I(\mathrm{bin}(f_j, B_j);\, y)$, i.e., the mutual information between the binned feature and the label.

\begin{definition}[Interaction Information]
\label{def:ii}
For discretized features $f_i, f_j$ and target $y$, the \emph{interaction information} is
\begin{equation}
\label{eq:ii}
\II(f_i, f_j) \;=\; I(f_i, f_j;\, y) \;-\; I(f_i;\, y) \;-\; I(f_j;\, y).
\end{equation}
When $\II > 0$, the pair exhibits \emph{synergy}: joint information exceeds the sum of marginal contributions. When $\II < 0$, the pair is \emph{redundant}.
\end{definition}

\begin{definition}[Interaction-only Signal]
	\label{def:ios}
	A feature pair $(f_i, f_j)$ exhibits an \emph{interaction-only signal} with respect to target $y$ if there exist thresholds $0 < \epsilon_1 \ll \epsilon_2$ such that:
	\begin{enumerate}[label=(\alph*),nosep]
		\item $I(f_i; y) < \epsilon_1$ and $I(f_j; y) < \epsilon_1$ \quad (weak marginal signal),
		\item $I(f_i, f_j; y) > \epsilon_2$ \quad (strong joint signal),
	\end{enumerate}
	where $\epsilon_1$ is small enough that neither feature would survive a marginal relevance screen and $\epsilon_2$ is large enough to be practically interesting.
	This implies $\II(f_i, f_j) > \epsilon_2 - 2\epsilon_1 \approx \epsilon_2$.
\end{definition}

\begin{definition}[High Utility Gain Pattern]
\label{def:hug}
A HUG pattern is an itemset $X$ in the top-$k$ high utility itemsets (constrained to maximum length $L$) that satisfies the information gain threshold: $\mathrm{HUG} = \{X \mid X \in \mathrm{topkHUI}_L \;\text{and}\; \IG(X) \geq G\}$, where the external utility of feature $f_j$ for class $c$ is defined as $\EU(f_j, c) = |\mathrm{corr}(f_j, y)| \cdot w_c$ and $w_c$ is a class weight.
\end{definition}

While we build upon HUG-IML framework \citep{Krishnamoorthy2024}, our contributions apply more broadly to any framework that screens features by marginal association. 

\begin{definition}[Pair Transform Family]
\label{def:aug}
For a selected feature pair $(f_i, f_j)$, the \emph{pair transform family} $\Phi_{ij}$ consists of four operators applied to the \emph{original continuous} feature values:
\begin{align}
\phi^{\mathrm{prod}}_{ij} &= f_i \cdot f_j, &
\phi^{\mathrm{abs}}_{ij} &= |f_i - f_j|, \label{eq:aug1}\\
\phi^{\mathrm{sign}}_{ij} &= f_i - f_j, &
\phi^{\mathrm{sum}}_{ij} &= f_i + f_j. \label{eq:aug2}
\end{align}
These transforms render specific interaction structures, namely multiplicative, distance-based, directional, and additive, accessible to a downstream linear classifier.
\end{definition}

\FloatBarrier
\section{Methodology}
\label{sec:methodology}

The proposed framework addresses the interaction gap in interpretable classification through a coordinated sequence of three major stages. Figure~\ref{fig:block} illustrates the overall architecture. The first stage, \emph{adaptive discretization} (Section~\ref{sec:adaptive}), replaces fixed global binning with a supervised per-feature bin selection that respects each feature's distributional complexity. The second stage, \emph{interaction detection and feature construction} (Section~\ref{sec:interaction}), applies an information-theoretic synergy criterion to identify feature pairs whose joint predictive value exceeds their marginal contributions and channels this evidence through one of two separable routing pathways. The third stage, \emph{complexity-budgeted learning} (Section~\ref{sec:budget}), caps the total number of explanation components via a partitioned budget and trains a sparse logistic regression classifier within this budget. The key design principle throughout is \emph{separable interpretability}: each additional representational capacity enters the model through a single, semantically transparent mechanism, so that every prediction can be unambiguously attributed to its source components.

\par IAIML first learns adaptive discretization rules using only the training fold. It then computes coarse-grid pairwise interaction scores on the same training fold and records the top source--partner pairs. \IAR{} uses these pairs only to relax candidate admission for pattern mining; \IAA{} instead converts them into named pair transforms and appends them to the sparse classifier vocabulary. In both pathways, the final classifier is fit after applying the partitioned component budget. All learned bins, selected pairs, mined patterns, and coefficients are then applied to the held-out outer fold.

\begin{figure}[t]
\centering
\includegraphics[width=0.6\textwidth,height=0.55\textheight]{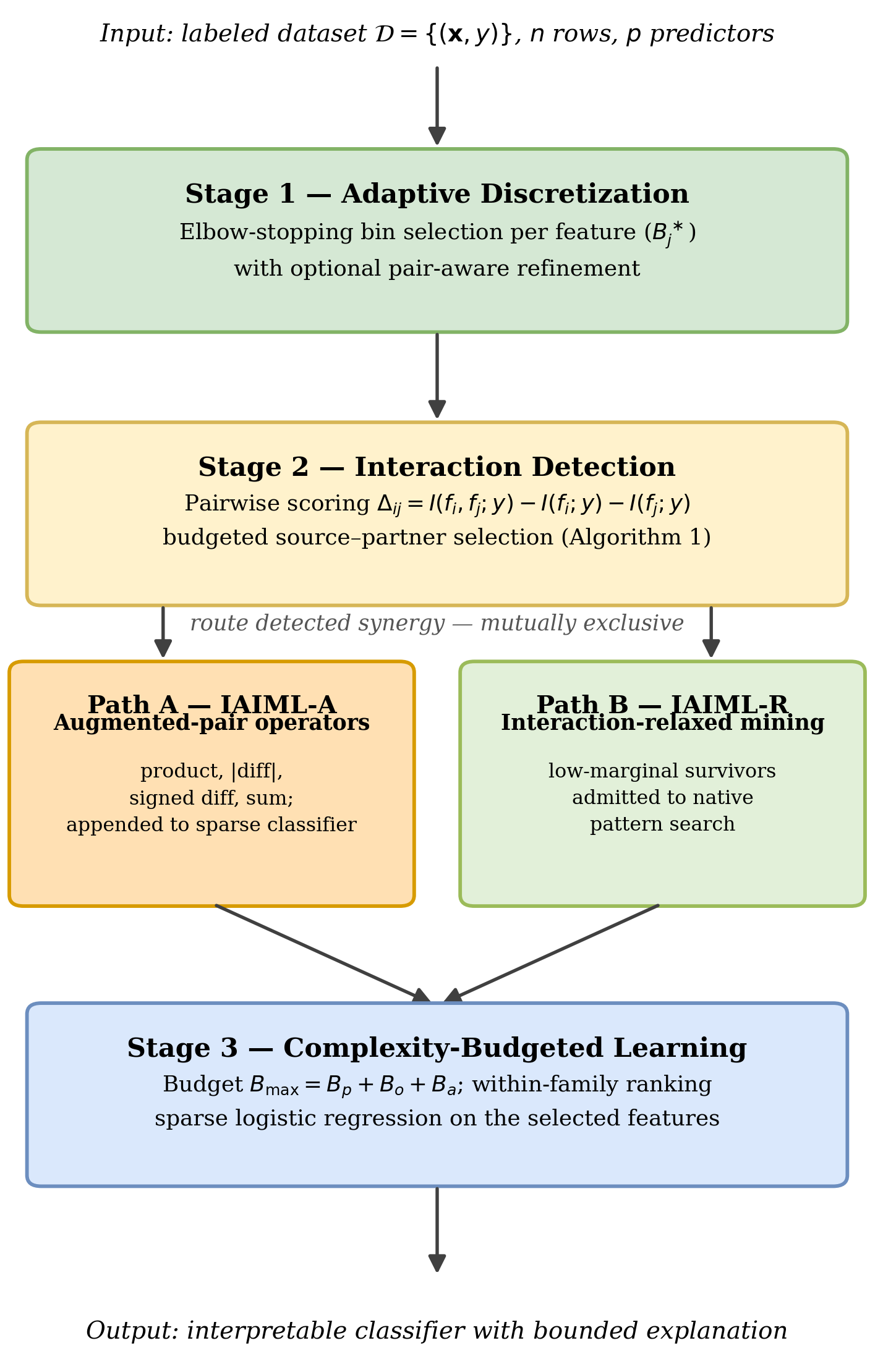}
\caption{High-level architecture of the interaction-aware interpretable classification framework. Stage 1 (green) performs adaptive per-feature discretization. Stage 2 (yellow) scores pairwise interaction information and selects source--partner pairs (Algorithm~\ref{alg:source}). The framework then routes detected synergy through exactly one of two mutually exclusive pathways, augmented-pair operators (orange, \IAA{}, Path A) or interaction-relaxed mining (light green, \IAR{}, Path B). Stage 3 (blue) applies a partitioned complexity budget and fits a sparse logistic regression classifier.}
\label{fig:block}
\end{figure}

\subsection{Adaptive Per-Feature Discretization}
\label{sec:adaptive}

Pattern-based interpretable classifiers that operate on discretized features and more broadly any itemset-mining or rule-induction framework that converts continuous predictors to categorical items, typically use equal-frequency binning with a global bin count $B$ shared across all numerical features. This is a deliberate interpretability choice: all mined items are stated in the same type of discretized vocabulary. It is also a source of avoidable error. A fixed global $B$ makes the explanation vocabulary easy to describe, but it asks every feature to express its class relationship at the same resolution. This creates two systematic failure modes.
\emph{Under-resolution} occurs when features with complex or multi-modal conditional distributions $p(f_j \mid y)$ are discretized too coarsely, merging distinct distributional regimes into single bins and destroying class-discriminative structure. \emph{Over-fragmentation} occurs when features with simple distributions or few distinct values are discretized too finely, producing bins with few observations and inflating estimation variance in downstream information-gain computations. The standard remedy, grid search over a small set of candidate values $B \in \{5, 7, 10, 15\}$, selects the best single global value but cannot resolve per-feature mismatches. A feature whose conditional distribution is bimodal may require $B = 10$ to separate the modes, while a near-binary feature needs only $B = 3$.

\subsubsection{Elbow-Point Information Gain Search}
We select $B_j$ per feature~$j$ via a supervised search over a candidate set $\mathcal{B} = \{B^{(1)} < B^{(2)} < \cdots < B^{(C)}\}$ (default $\{2,3,5,7,10,15\}$) with an elbow-stopping rule.

\begin{proposition}[Monotonicity of Training-Sample IG]
	\label{prop:mono}
	Let $\hat{P}_n$ denote the empirical distribution of $(f_j, y)$
	over $n$ training samples. If the partition $\pi_{B'}$ induced by $B'$ equal-frequency bins refines $\pi_B$ for $B' > B$, then $\IG(y;\,\mathrm{bin}(f_j, B')) \geq \IG(y;\,\mathrm{bin}(f_j, B))$, where both quantities are computed under $\hat{P}_n$.
\end{proposition}

\textit{Proof.}
Refinement implies $\mathrm{bin}(f_j,B) = g(\mathrm{bin}(f_j,B'))$ for some surjection~$g$, so $y \to \mathrm{bin}(f_j,B') \to \mathrm{bin}(f_j,B)$ is a Markov chain and the data-processing inequality gives the result. \qed

For arbitrary candidate counts the quantile cutpoints of $\pi_B$ and $\pi_{B'}$ need not coincide, so strict nesting does not hold in general. In practice the IG sequence is nevertheless observed to be non-decreasing on training data: quantile bins become approximately nested for continuous features with large~$n$, and any minor non-monotonicity simply triggers early stopping rather than selecting a suboptimal bin count. Exact nesting can be enforced by choosing $B^{(c)} \mid B^{(c+1)}$ (e.g.\ $\{2,4,8,16,32\}$).

On held-out data, however, high bin counts may \emph{decrease} IG due to overfitting. The elbow-stopping rule resolves this without requiring a held-out set.

\begin{definition}[Elbow-Stopping Rule]
	\label{def:elbow}
	Given the candidate IG sequence $\{\IG_j(B^{(c)})\}_{c=1}^{C}$, marginal gain ratio threshold $\rho \in (0,1)$, and stabilisation constant $\epsilon_0 = 10^{-9}$, define the \emph{first-violation index}
	\begin{equation}
		\label{eq:tau}
		\tau_j \;\coloneqq\; \min\!\left\{\, c \in \{2,\ldots,C\} \;:\; \frac{\IG_j(B^{(c)}) - \IG_j(B^{(c-1)})}{\IG_j(B^{(c-1)}) + \epsilon_0} \;<\; \rho \;\right\},
	\end{equation}
	with $\tau_j \coloneqq C + 1$ if no such $c$ exists. The selected bin count is
	\begin{equation}
		\label{eq:elbow}
		B_j^* \;=\; B^{(\tau_j - 1)}.
	\end{equation}
\end{definition}

Since $\tau_j \geq 2$ by construction, $B_j^* \geq B^{(1)}$: the coarsest candidate is always admissible, providing a natural fallback for uninformative features. When $\IG_j(B^{(1)}) = 0$, the ratio at $c = 2$ reduces to $\IG_j(B^{(2)})/\epsilon_0$, which exceeds $\rho$ for any feature with non-trivial class separation at the second candidate.

\begin{proposition}[Properties of Elbow-Stopping]
	\label{prop:elbow}
	\begin{enumerate}[label=(\alph*),nosep]
		\item $B_j^* \leq B^{(C)}$, bounding the cardinality of the discretised feature.
		\item If $\IG_j(\cdot)$ is concave over $\mathcal{B}$ and the
		candidates are equally spaced, the marginal gain ratio is
		non-increasing, so elbow-stopping selects a bin count in
		the knee region of the IG curve.
		\item $B_j^* \leq \min(B^{(C)},\, d_j)$ where $d_j \coloneqq |\{x_j^{(i)}\}_i|$: for $B > d_j$ the implementation collapses duplicate quantile edges, the effective partition is unchanged, the marginal gain drops to zero, and the rule stops.
	\end{enumerate}
\end{proposition}

A na\"ive implementation evaluates each $(j,B)$ pair independently at cost $O(C \cdot n \log n)$ per feature. Our implementation sorts once and reuses the sorted order for all candidates, reducing the cost to $O(n \log n + C \cdot n)$ per feature and $O(p\,(n \log n + C \cdot n))$ overall.

\subsubsection{Pair-Aware Bin Selection}
\label{sec:pair_aware}

The marginal procedure can be suboptimal for features whose predictive
value is interaction-driven.  When interaction scoring
(Section~\ref{sec:interaction}) has identified a feature's partners
$\mathcal{P}(j)$, we define the \emph{pair score} for candidate bin
count~$B$ applied to both $f_j$ and partner $f_q$:
\begin{equation}
	\label{eq:ps}
	\mathrm{ps}(f_j, f_q, B)
	\;=\; \max\!\Big(\;
	I\!\big(\mathrm{bin}(f_j,B);\, y \mid \mathrm{bin}(f_q,B)\big),\;\;
	\max\!\big(\IG_j(B),\; \IG_q(B)\big)
	\;\Big).
\end{equation}
The first branch is the conditional mutual information, the unique
signal $f_j$ contributes beyond its partner.  The second ensures the
score is at least as large as either marginal contribution.  Both are
non-negative and bounded above by
$I(\mathrm{bin}(f_j,B), \mathrm{bin}(f_q,B);\,y)$.

The \emph{pair-aware score} for feature~$j$ is
\begin{equation}
	\label{eq:pair_score}
	S_j(B) \;=\; \max\!\Big(\,\IG_j(B),\;\;
	\max_{f_q \in \mathcal{P}(j)}\; \mathrm{ps}(f_j, f_q, B)\,\Big).
\end{equation}
The selection applies a \emph{coarsen-by-default} policy: candidates are
capped at $B \leq B_{\mathrm{cap}} = 8$ (provided at least one
candidate satisfies this bound; otherwise all candidates are retained)
to keep joint contingency tables tractable
($B_{\mathrm{cap}}^2 = 64$ cells), and the smallest~$B$ whose score
reaches at least $\alpha = 0.85$ of the best is selected:
\begin{equation}
	\label{eq:coarsen}
	B_j^* \;=\; \min\!\bigl\{\, B \in \mathcal{B}_{\leq B_{\mathrm{cap}}}
	\;:\; S_j(B) \;\geq\; \alpha \cdot \max_{B'}\, S_j(B') \,\bigr\}.
\end{equation}
Features without interaction partners fall back to the marginal
elbow-stopping rule (Definition~\ref{def:elbow}).

\begin{remark}[Interpretability of Adaptive Resolution]
	\label{rem:pair_aware}
	Adaptive binning does not alter the explanation vocabulary: the final
	items remain statements about original variables falling into intervals;
	only the number of intervals is learned per variable.  Pair-aware
	binning prevents an interaction-relevant variable from being discretised
	at a resolution chosen solely from a flat marginal IG curve, without
	introducing an interaction feature itself.  The coarsen-by-default
	policy and the cap $B \leq 8$ ensure that resolution increases only
	when the training data support it.
\end{remark}

\subsection{Interaction Information Scoring and Feature Construction}
\label{sec:interaction}

Definitions~\ref{def:ii}--\ref{def:ios} characterise interaction-only
feature pairs: pairs whose predictive content resides almost entirely in
their joint distribution.  This section establishes why standard
marginal-selection methods miss such pairs
(Section~\ref{sec:limitation}), connects the interaction information to
the Williams--Beer synergy decomposition
(Section~\ref{sec:synergy_measure}), and describes the budgeted scoring
algorithm that identifies synergistic sources for downstream feature
construction (Section~\ref{sec:source_alg}).
Section~\ref{sec:interaction_paths} specifies how the resulting
interaction evidence is routed into the model.


\subsubsection{Interaction Information as a Synergy Measure}
\label{sec:synergy_measure}

The interaction information (Definition~\ref{def:ii}) directly measures
the excess joint predictive power of a feature pair beyond marginal
contributions.  We establish its connection to PID synergy under a
minimal discreteness assumption.

\begin{assumption}[Discrete Feature Space]
	\label{assum:discrete}
	Features $f_i, f_j$ take values in finite alphabets
	$\mathcal{V}_i, \mathcal{V}_j$ of sizes $B_i, B_j$, and the target $y$
	takes values in a finite label set of size~$K$.  Zero-probability cells
	are permitted; entropy terms use the convention $0\log 0 = 0$.
\end{assumption}

\begin{proposition}[Connection to PID Synergy]
	\label{prop:synergy}
	Under Assumption~\ref{assum:discrete}, the interaction information
	relates to the Williams--Beer Partial Information Decomposition
	\citep{Williams2010} as
	\begin{equation}
		\label{eq:pid}
		\II(f_i, f_j) \;=\;
		\mathrm{Syn}(f_i, f_j \to y)
		\;-\;
		\mathrm{Red}(f_i, f_j \to y),
	\end{equation}
	where $\mathrm{Syn}$ is the synergistic information and $\mathrm{Red}$
	is the redundant (shared) information.
\end{proposition}

\begin{proof}
	The PID decomposes joint mutual information into four non-negative atoms:
	\begin{equation}
		I(f_i, f_j;\, y) \;=\;
		\mathrm{Uniq}(f_i) + \mathrm{Uniq}(f_j)
		+ \mathrm{Red} + \mathrm{Syn},
	\end{equation}
	with marginal identities
	$I(f_k;\,y) = \mathrm{Uniq}(f_k) + \mathrm{Red}$ for
	$k \in \{i,j\}$.  Substituting into~\eqref{eq:ii}:
	\begin{align}
		\II(f_i, f_j)
		&= \bigl[\mathrm{Uniq}(f_i) + \mathrm{Uniq}(f_j) + \mathrm{Red}
		+ \mathrm{Syn}\bigr]
		- \bigl[\mathrm{Uniq}(f_i) + \mathrm{Red}\bigr]
		- \bigl[\mathrm{Uniq}(f_j) + \mathrm{Red}\bigr]
		\nonumber \\
		&= \mathrm{Syn} - \mathrm{Red}. \qedhere
	\end{align}
\end{proof}

\begin{corollary}[Interaction Information for Interaction-Only Pairs]
	\label{cor:ios}
	For an interaction-only pair (Definition~\ref{def:ios}),
	$\mathrm{Uniq}(f_i) \approx 0$, $\mathrm{Uniq}(f_j) \approx 0$, and
	$\mathrm{Red} \leq \min(I(f_i;\,y),\, I(f_j;\,y)) < \epsilon_1$.
	Therefore $\II(f_i, f_j) \approx \mathrm{Syn}(f_i, f_j \to y)$: the
	interaction information is an especially clean proxy for synergy in the
	interaction-only regime.
\end{corollary}

\begin{remark}[PID Framework Dependence]
	The identity $I(f_k;\,y) = \mathrm{Uniq}(f_k) + \mathrm{Red}$ is an
	axiom of the Williams--Beer framework \citep{Williams2010}; alternative
	PID definitions \citep{Griffith2014, Bertschinger2014} may yield
	different decompositions.  However, $\II$ itself is a well-defined
	information-theoretic quantity, a signed combination of standard
	mutual-information terms, and does not depend on any PID framework.
	Only its interpretation as $\mathrm{Syn} - \mathrm{Red}$ is
	framework-dependent; the ranking criterion based on $\II$ is not.
\end{remark}


\subsubsection{Limitations of Marginal Feature Selection}
\label{sec:limitation}

With the synergy characterisation in place, we identify the structural
mechanism by which interaction-only pairs are suppressed in any
selection pipeline that ranks features by marginal target association.

\begin{assumption}[Marginal-Utility Coupling]
	\label{assum:utility}
	The utility assigned to feature~$f_j$ for class~$c$ satisfies
	$\EU(f_j, c) = g\bigl(I(f_j;\,y)\bigr) \cdot w_c$, where
	$g\colon \R_{\geq 0} \to \R_{\geq 0}$ is continuous, non-decreasing,
	with $g(0)=0$, and $w_c > 0$ is a class weight.  This captures any
	utility that is a continuous monotone function of mutual information.
	Correlation-based utilities such as that in~\citep{Krishnamoorthy2024}
	satisfy the weaker property that $I(f_j;\,y) \to 0$ implies
	$\mathrm{corr}(f_j,y) \to 0$ (since vanishing MI implies independence),
	so the conclusion of Proposition~\ref{thm:hug_limitation} applies to
	them as well.
\end{assumption}

\begin{proposition}[Marginal-Utility Suppression]
	\label{thm:hug_limitation}
	Under Assumption~\ref{assum:utility}, let $(f_i, f_j)$ be an
	interaction-only pair with $I(f_i;\,y),\, I(f_j;\,y) < \epsilon_1$.
	\begin{enumerate}[label=(\alph*),nosep]
		\item \emph{Utility suppression.}
		For each constituent feature,
		$\EU(f_k, c) \leq g(\epsilon_1) \cdot w_c \to 0$ as
		$\epsilon_1 \to 0$.  Any top-$k$ selection stage that ranks
		items by a non-decreasing function of marginal feature--target
		association therefore assigns vanishing priority to both
		features before their joint value is evaluated.
		
		\item \emph{Binarisation loss.}
		The binary indicator $\mathbb{1}[\text{pattern present}]$ is a
		deterministic function of $(f_i,f_j)$, so
		$I(\mathbb{1}[\text{pattern}];\,y) \leq I(f_i,f_j;\,y)$ by the
		data-processing inequality.  This many-to-one mapping collapses
		the $B_i \times B_j$ joint configuration space to a single bit,
		discarding the within-pattern variation on which
		synergistic signal depends.
	\end{enumerate}
\end{proposition}

\begin{proof}
	Part~(a): $g$ is continuous and non-decreasing with $g(0)=0$, so
	$I(f_k;\,y) < \epsilon_1$ gives
	$\EU(f_k,c) = g(I(f_k;\,y)) \cdot w_c \leq g(\epsilon_1) \cdot w_c
	\to 0$ as $\epsilon_1 \to 0$.
	Part~(b): the data-processing inequality gives the bound.  The
	qualitative claim follows from the observation that a binary variable
	has entropy at most one bit, whereas the interaction-only pair's
	predictive content, distributed across the full joint contingency
	table, can require $\log_2 \min(B_i B_j, K)$ bits to represent.
\end{proof}

\begin{remark}[Scope]
	\label{rem:scope}
	This result identifies a structural tendency in marginal-selection
	methods, not a peculiarity of any specific algorithm.  The limitation is
	most direct for pattern-mining frameworks that apply marginal utilities
	before evaluating candidate itemsets.  EBM's greedy FAST pair selection
	partially mitigates this by examining residuals after main-effect
	fitting, but if both features contribute negligible main effects their
	pairwise signal competes with $\binom{p}{2}$ candidates under a fixed
	interaction budget.  RuleFit's tree-derived rules can implicitly capture
	interactions, but only those that the greedy splitting heuristic happens
	to explore.  Section~\ref{sec:experiments} evaluates how much these
	implicit mechanisms recover in practice.
\end{remark}


\subsubsection{Source Selection Algorithm}
\label{sec:source_alg}

Algorithm~\ref{alg:source} formalises the interaction-aware source
selection as a budgeted best-partner ranking procedure.  Each source
column is scored by the strongest interaction information observed with
any eligible partner, and the top~$s$ sources are retained.

\begin{algorithm}[t]
	\caption{Budgeted Best-Partner Interaction Source Selection}
	\label{alg:source}
	\begin{algorithmic}[1]
		\Require Feature matrix $X \in \R^{n \times p}$, labels
		$y \in \{1,\ldots,K\}^n$, source budget~$s$, coarse bin
		count $B_c = 4$, optional partner-search budget~$r$
		\Ensure Ranked source records
		$(j,\;\mathrm{partner}(j),\;\widehat{\Theta}_j)$
		\State \textbf{Discretise:} For each eligible column~$j$, compute
		$c_j \gets \mathrm{qbin}(f_j, B_c)$
		\Comment{Equal-frequency bins; non-finite rows skipped}
		\State \textbf{Marginal MI:} For each~$j$, compute
		$\widehat{\IG}_j \gets I(c_j;\,y)$ from the
		$B_c \times K$ contingency table
		\State Initialise $\widehat{\Theta}_j \gets 0$,\;
		$\mathrm{partner}(j) \gets \varnothing$ for every eligible~$j$
		\If{partner-search budget $r$ is specified}
		\State Choose a deterministic subset of $r$ partner anchors and
		score all pairs with at least one anchor
		\Else
		\State Score all $\binom{p}{2}$ unordered pairs
		\EndIf
		\For{each scored pair $(a,b)$}
		\State $\widehat{\IG}_{ab} \gets I(c_a, c_b;\,y)$ from the
		$B_c^2 \times K$ contingency table
		\State $\widehat{\Delta}_{ab} \gets
		\widehat{\IG}_{ab} - \widehat{\IG}_a - \widehat{\IG}_b$
		\If{$\widehat{\Delta}_{ab} > \widehat{\Theta}_a$}
		$\;\widehat{\Theta}_a \gets \widehat{\Delta}_{ab}$;\;
		$\mathrm{partner}(a) \gets b$
		\EndIf
		\If{$\widehat{\Delta}_{ab} > \widehat{\Theta}_b$}
		$\;\widehat{\Theta}_b \gets \widehat{\Delta}_{ab}$;\;
		$\mathrm{partner}(b) \gets a$
		\EndIf
		\EndFor
		\State Rank by $\widehat{\Theta}_j$ descending, breaking ties by
		marginal information then by deterministic index order
		\State \Return the top $s$ source records with their best partners
	\end{algorithmic}
\end{algorithm}

\begin{remark}[Coarse Grid for Interaction Estimation]
	\label{rem:coarse}
	The coarse grid $B_c = 4$ is deliberate: the joint contingency table has
	only $B_c^2 K = 16K$ cells, yielding more reliable per-cell frequencies
	than finer grids at the same sample size.  Since interaction scores
	serve as a ranking criterion for source selection rather than as
	calibrated downstream features, this resolution is appropriate.
	Theorem~\ref{thm:sample} makes the trade-off precise: the
	concentration bound scales with $B_i B_j K$, so coarser grids give
	tighter control at fixed~$n$.
\end{remark}


\subsubsection{Routing Pathways for Interaction Evidence}
\label{sec:interaction_paths}

Interaction information is evidence, not a representation.  Once a pair
receives a high synergy score, the model must specify how that evidence
enters the explanation.  IAIML provides two mutually exclusive pathways.

\paragraph{Interaction-Relaxed Mining (\IAR).}
Survivor columns identified by Algorithm~\ref{alg:source} are admitted
into the native pattern-mining candidate pool, even if their marginal
utility alone would not qualify them.  No generated features are added:
the downstream classifier sees the same kinds of objects as before (original features and mined patterns) but the candidate pattern space
is broader.  This is appropriate when the final explanation should
remain rule-like (``these original conditions co-occur'') rather than
algebraic.  The cost is that the interaction is implicit in which
patterns were allowed to grow; understanding why a weak-marginal
variable entered the model requires inspecting the mining provenance.

\paragraph{Augmented-Pair Operators (\IAA).}
Each selected source--partner pair is converted into named continuous
operators (Section~\ref{sec:aug_gen}) and appended to the downstream
linear classifier.  The explanation becomes a sparse linear model over
original features, patterns, and explicit pair terms.  Each added
column has a defined algebraic meaning and a fitted coefficient, so
the model remains inherently interpretable, though less rule-like than
relaxed mining.  The advantage is directness: if a multiplicative or
distance-based interaction drives the target, the model assigns weight
to that named interaction rather than relying on a binary itemset to
approximate the geometry.

\medskip\noindent
The two pathways are mutually exclusive: enabling both would let the same
interaction signal alter both the upstream pattern search and the
downstream feature matrix, weakening attribution.


\subsubsection{Augmented Feature Generation}
\label{sec:aug_gen}

For each selected pair $(f_i, f_j) \in \mathcal{S}$, the four operators
from Definition~\ref{def:aug} are applied to the \emph{original
	continuous} feature values, producing up to $4|\mathcal{S}|$ augmented
features.  Operating on continuous values rather than discretised
representations preserves within-bin variation that carries the
interaction signal.

Each operator linearises a specific interaction structure.  Writing
$h\colon \R^2 \to \R$ for the true interaction function relating
$(f_i,f_j)$ to the log-odds of $y=1$:
$\phi^{\mathrm{prod}}_{ij}$ covers multiplicative effects
($h \propto ab$),
$\phi^{\mathrm{abs}}_{ij}$ covers distance-based effects
($h \propto |a-b|$),
$\phi^{\mathrm{sign}}_{ij}$ covers directional effects
($h \propto a-b$), and
$\phi^{\mathrm{sum}}_{ij}$ covers additive effects
($h \propto a+b$).
In each case $h(a,b) = \beta\cdot\phi(a,b)$, so a linear classifier
with coefficient~$\beta$ recovers~$h$ exactly.  The operator set is
deliberately small: it is not intended to span all pairwise functions.
Threshold- and parity-style interactions remain better represented by
discretised patterns or tree-derived rules. 

\subsection{Complexity-Budgeted Learning}
\label{sec:budget}

\IAA{} models combine three feature families (mined patterns, original
features, and augmented interaction operators), so the effective
dimensionality $p_{\text{eff}} = |P_{\text{mined}}| + |F_{\text{orig}}|
+ |F_{\text{aug}}|$ is no longer controlled by a single parameter.  An
unconstrained $p_{\text{eff}}$ undermines both interpretability and
generalisation.  We therefore impose a budget governed by a single
parameter $K$ (\texttt{topK}).

\begin{definition}[Feature Budget]
	\label{def:budget}
	Let $B_p$, $B_o$, $B_a$ denote the number of retained mined patterns,
	original features, and augmented interaction operators respectively.
	Two budget modes are provided:
	\begin{enumerate}[label=(\alph*),nosep]
		\item \emph{Per-family} (default).  Each family is independently
		capped at~$K$: $B_p \leq K$, $B_o \leq K$, $B_a \leq K$, giving
		$|\Phi_{\mathrm{final}}| \leq 3K$.
		\item \emph{Strict}.  A single global cap
		$|\Phi_{\mathrm{final}}| \leq K$ is enforced by pooling all
		features into a unified ranking and retaining the top~$K$
		regardless of family membership.
	\end{enumerate}
\end{definition}

Within each family, features are ranked by an appropriate relevance
criterion and the top entries are retained: mined patterns by
information gain $\IG(X)$; original features by adaptive-binning
information gain $\IG_j(B_j^*)$; augmented operators by target
correlation $|\mathrm{corr}(\phi_{ij}^{(\cdot)},\, y)|$.  In strict
mode the per-family rankings are computed first, then a normalised
score (zero mean, unit variance within each family) determines the
global selection.


\FloatBarrier
\section{Theoretical Guarantees}
\label{sec:theory}

This section provides conditional guarantees for the discretized
screening and budgeted representation used by IAIML.  All results
operate on the finite-alphabet plug-in mutual-information estimator
applied after binning and none of them require a lower bound on cell
probabilities.  The concentration and ranking results
(Sections~\ref{sec:sample}--\ref{sec:ranking}) are distribution-free.
The generalization bound (Section~\ref{sec:gen}) treats the dictionary
size~$N$ and budget~$\Bmax$ as fixed.  All four results concern
pairwise interaction structure.

We address four questions: how empirical interaction scores concentrate
as a function of contingency-table size (Section~\ref{sec:sample}),
when this suffices for stable source selection
(Section~\ref{sec:ranking}), what generalization guarantee the
complexity budget provides (Section~\ref{sec:gen}), and the
computational cost of the full pipeline (Section~\ref{sec:comp}).

\subsection{Finite-Sample Concentration of Interaction Scores}
\label{sec:sample}

The interaction-scoring stage ranks source variables by their strongest
partner scores.  The relevant finite-sample question is how the
empirical interaction score concentrates as a function of the
contingency-table size.

Let
$\Delta_{ij} = I((f_i,f_j);\,y) - I(f_i;\,y) - I(f_j;\,y)$
denote the population interaction score for a binned pair and let
$\widehat{\Delta}_{ij}$ be the plug-in estimate.  For an alphabet of
size~$a$ define the entropy modulus
\begin{equation}
	\omega_a(\eta)
	= \eta\log(a-1) + h_2(\eta),
	\qquad 0 < \eta < 1-1/a,
\end{equation}
where $h_2$ is the binary entropy function.  This is a standard
Fannes-type continuity term \citep{Fannes1973} that handles
zero-probability cells through total-variation continuity.

\begin{theorem}[Uniform Concentration for Discretized Interaction
	Scores]
	\label{thm:sample}
	Let $\mathcal{Q}$ be a set of $m$ candidate feature pairs whose binned
	pair-label tables each have at most
	$M = \max_{(i,j)\in\mathcal{Q}} B_i B_j K$ cells.  If
	$0<\eta<1-1/M$ and
	\begin{equation}
		\label{eq:l1_sample}
		n \;\geq\; \frac{2}{\eta^2}
		\bigl(M\log 2 + \log(m/\delta)\bigr),
	\end{equation}
	then with probability at least $1-\delta$,
	\begin{equation}
		\label{eq:uniform_delta}
		\max_{(i,j)\in\mathcal{Q}}
		|\widehat{\Delta}_{ij}-\Delta_{ij}|
		\;\leq\; 7\,\omega_M(\eta).
	\end{equation}
\end{theorem}

\begin{proof}
	For a fixed pair the empirical distribution over $(f_i,f_j,y)$ is
	multinomial on at most $M$ cells.  The bound of
	\citet{Weissman2003} gives
	\begin{equation}
		\Pr\!\bigl(\|\widehat{P}_{ijy}-P_{ijy}\|_1 > \eta\bigr)
		\leq \exp(M\log 2 - n\eta^2/2).
	\end{equation}
	A union bound over $m$ pairs yields
	equation~\eqref{eq:l1_sample}.  Marginalization cannot increase
	$\ell_1$ distance so the same event controls all marginals needed
	below.
	
	Expanding $\Delta_{ij}$ in entropy terms:
	\begin{equation}
		\Delta_{ij}
		= H(f_i,f_j) - H(f_i) - H(f_j) - H(y)
		+ H(f_i,y) + H(f_j,y) - H(f_i,f_j,y).
	\end{equation}
	This involves seven distinct entropy terms each over an alphabet of
	size at most~$M$.  Since $\omega_a$ is increasing in~$a$, Fannes'
	inequality gives
	$|\widehat{H}(\cdot)-H(\cdot)| \leq \omega_M(\eta)$ per term and
	the triangle inequality yields
	$|\widehat{\Delta}_{ij}-\Delta_{ij}| \leq 7\,\omega_M(\eta)$.
\end{proof}

The dominant quantity is $M = B_i B_j K$: stable interaction scores
require data proportional to the joint contingency-table size.  This
motivates the coarse $B_c = 4$ grid for interaction screening and the
cap $B_{\mathrm{cap}} = 8$ in pair-aware bin selection.

\subsection{Source-Ranking Consistency}
\label{sec:ranking}

Each source feature receives the best interaction score achieved with
any eligible partner.  For feature~$j$ define the population
best-partner score
\begin{equation}
	\Theta_j
	= \max\!\bigl\{0,\;
	\max_{q:(j,q)\in\mathcal{Q}} \Delta_{jq}\bigr\},
\end{equation}
and let $\widehat{\Theta}_j$ be the empirical analogue.  Order the
population scores as
$\Theta_{(1)} \geq \Theta_{(2)} \geq \cdots \geq \Theta_{(p')}$.

\begin{theorem}[Best-Partner Source Ranking Consistency]
	\label{thm:ranking}
	Let $S_s^\star$ and $\widehat{S}_s$ be the population and empirical
	top-$s$ source sets.  If the source-selection margin
	$\Gamma_s = \Theta_{(s)} - \Theta_{(s+1)} > 0$ satisfies
	\begin{equation}
		\label{eq:ranking_eta}
		7\,\omega_M(\eta) \;\leq\; \Gamma_s/2
	\end{equation}
	with $n$ satisfying equation~\eqref{eq:l1_sample}, then
	$\widehat{S}_s = S_s^\star$ with probability at least $1-\delta$.
\end{theorem}

\begin{proof}
	On the event of Theorem~\ref{thm:sample} every pair score is
	estimated within $7\omega_M(\eta)$.  Taking a maximum over partners
	is non-expansive in the sup norm so
	$\max_j |\widehat{\Theta}_j - \Theta_j| \leq 7\omega_M(\eta)$.
	If this is at most $\Gamma_s/2$ then for all
	$a \in S_s^\star$ and $b \notin S_s^\star$:
	\begin{equation}
		\widehat{\Theta}_a - \widehat{\Theta}_b
		\;\geq\;
		(\Theta_a - \Gamma_s/2) - (\Theta_b + \Gamma_s/2)
		\;\geq\; 0. \qedhere
	\end{equation}
\end{proof}

\begin{corollary}[Threshold Selection]
	\label{cor:threshold}
	Fix a threshold $\tau$ and margin $\gamma > 0$.  If
	$7\omega_M(\eta) \leq \gamma$ and equation~\eqref{eq:l1_sample}
	holds then with probability at least $1-\delta$ every source with
	$\Theta_j \geq \tau + \gamma$ is selected by the empirical rule
	$\widehat{\Theta}_j \geq \tau$ and every source with
	$\Theta_j \leq \tau - \gamma$ is rejected.
\end{corollary}

\begin{remark}[Operational Interpretation]
	The guarantee depends on the separation margin $\Gamma_s$ rather than
	on absolute estimation accuracy.  When interaction sources are well
	separated the selected set is stable even with imprecise individual
	estimates.  When $\Gamma_s$ is small the boundary is intrinsically
	ambiguous and bootstrap or permutation diagnostics are more informative
	than a single ranking.
\end{remark}

\subsection{Generalization Bound for Budgeted Models}
\label{sec:gen}

The following bound is conditional on a finite candidate dictionary
and isolates the cost of selecting at most $\Bmax$ components from it.

\begin{theorem}[Generalization with Feature Selection]
	\label{thm:generalization}
	Let $\mathcal{F}$ be the candidate dictionary with
	$|\mathcal{F}| = N$.  Consider linear classifiers supported on at
	most $\Bmax$ features from $\mathcal{F}$.  With probability at least
	$1-\delta$ over a sample of size~$n$ every such classifier~$h$
	satisfies
	\begin{equation}
		\label{eq:gen}
		R(h) \;\leq\; \hat{R}(h)
		\;+\; \sqrt{\frac{2(\Bmax+1)\log(en/(\Bmax+1))}{n}}
		\;+\; \sqrt{\frac{\log\!\bigl(2\binom{N}{\Bmax}/\delta\bigr)}{2n}}.
	\end{equation}
\end{theorem}

\begin{proof}
	For any fixed subset of $\Bmax$ features the VC-dimension of linear
	classifiers in $\R^{\Bmax}$ is $\Bmax + 1$
	\citep{Vapnik1995}.  Applying the VC bound with failure probability
	$\delta/\binom{N}{\Bmax}$ to each subset and union-bounding gives
	equation~\eqref{eq:gen}.
\end{proof}

The term $\log\binom{N}{\Bmax}$ is the statistical price of
data-dependent feature selection.  Reducing $\Bmax$ simultaneously
reduces the number of explanation components and this combinatorial
penalty.

\subsection{Computational Complexity}
\label{sec:comp}

Let $D$ denote the number of discretized atoms after preprocessing.
The total training cost decomposes into four stages:
\begin{equation}
	\label{eq:complexity}
	T_{\mathrm{total}}
	\;=\;
	\underbrace{O(pn\log n)}_{\text{adaptive binning}}
	\;+\;
	\underbrace{T_{\mathrm{mine}}(n,D,K)}_{\text{pattern mining}}
	\;+\;
	\underbrace{O(nm_{\mathrm{pair}})}_{\text{interaction scoring}}
	\;+\;
	\underbrace{T_{\mathrm{LR}}(n,\Bmax)}_{\text{downstream classifier}}.
\end{equation}

Adaptive binning sorts each of the $p$ raw features once and evaluates a
fixed set of candidate bin counts, giving cost $O(pn\log n)$.  Pattern
mining at depth $L{=}1$ performs a single linear scan over $n$ examples
and $D$ atoms.  At $L{=}2$ the search space is nominally
$\binom{D}{2}$ atom pairs, but the top-$K$ budget serves as a dynamic
pruning threshold: the utility of the $K$-th best pattern found so far
prunes any candidate whose upper-bound utility falls below it, so the
number of pairs actually evaluated is $C_2 \leq \binom{D}{2}$, with
$C_2$ decreasing as $K$ or $G$ tightens.  At higher depths the same
upper-bound pruning controls the combinatorial growth, making $K$ the
primary lever for bounding mining cost.  Each interaction pair requires
one $O(n)$ contingency-table pass, giving interaction cost
$O(nm_{\mathrm{pair}})$ with $m_{\mathrm{pair}} = \Theta(p^2)$ under
exhaustive scoring and $m_{\mathrm{pair}} = O(pr)$ under a partner
budget of size~$r$.  The downstream logistic-regression cost
$T_{\mathrm{LR}}(n,\Bmax)$ is governed by repeated passes over the
$n \times \Bmax$ design matrix.

\FloatBarrier
\section{Experimental Results}
\label{sec:experiments}

Our empirical study evaluates IAIML against tuned gradient-boosted tree ensembles and two interpretable interaction-aware baselines. The dataset panel contains 40 binary classification tasks: 24 real-world tabular benchmarks spanning credit risk, marketing response, health, and pricing domains (full list and sizes in Appendix~\ref{app:datasets}, Table~\ref{tab:validation_datasets}), and 16 synthetic interaction stress tests designed to exhibit specific pairwise and higher-order interaction structures (XOR, correlated-masked pairs, high-dimensional buried pairs, categorical--continuous conjunctions, threshold steps, parity groups, and other controlled regimes).

Six models are compared. {\IAA{}} and {\IAR{}} instantiate the two routing pathways of Section~\ref{sec:interaction_paths}. {XGBoost} and {LightGBM} are tuned gradient-boosted tree ensembles serving as strong predictive baselines. {EBM} (Explainable Boosting Machine) and {RuleFit} are interpretable interaction-aware baselines representative of the GAM-with-interactions and tree-derived rule-ensemble families.

Hyperparameters are selected via a \emph{fully nested} cross-validation protocol. For each dataset, an outer stratified 5-fold split provides the train/test partitions used for performance estimation; within each outer training fold, an inner stratified 3-fold cross-validation loop (over the grids detailed in Appendix~\ref{app:hyperparams}) selects hyperparameters independently, and the selected configuration is refit on the full outer training fold and scored on the held-out outer test fold. This nested design ensures that no information from a test fold influences the hyperparameter choice used to score it.

\textbf{Complexity metric.} We measure the number of explanation components using an audit-load proxy: the number of atomic fitted units that a reviewer would need to inspect to characterize the fitted model at implementation level. For IAIML, these are selected mined patterns, original features, and augmented pair terms; for XGBoost and LightGBM, realized terminal leaves; for EBM, active lookup-table cells; and for RuleFit, nonzero rule or linear terms. These units are not equally expressive across families: an IAIML pattern of length $L=2$ encodes a conjunction of two discretized conditions with one coefficient, while a boosted ensemble's feature effects are scattered across many leaves, and an EBM lookup-table cell provides fine-grained nonlinear resolution at the cost of more entries. The proxy is therefore an approximation of audit load; description-length measures or human-subjects time-to-decision studies would refine the comparison \citep{Rudin2019}. Despite this caveat, component count provides a useful and reproducible order-of-magnitude comparison across families.

Primary performance metrics are AUC, F1, and accuracy, averaged over the outer folds. Statistical comparison uses per-dataset mean AUC, followed by Wilcoxon signed-rank tests with Holm correction across the five pairwise comparisons against \IAA{}, and a Friedman test across all six models with mean-rank reporting.

\subsection{Main Results}
\label{sec:main_results}

\begin{table}[t]
\centering
\caption{Nested cross-validation results, sorted by mean AUC within each block. The overall block averages over the full 40-dataset panel; the real-world and synthetic blocks repeat the same statistics within each of the panel's two components (24 and 16 datasets respectively). AUC, F1, accuracy, and complexity are averaged over all outer folds.}
\label{tab:main_validation}
\small
\resizebox{\textwidth}{!}{%
\begin{tabular}{lrrrrrr}
\toprule
\textbf{Model} & \textbf{Mean AUC} & \textbf{Std.\ AUC} & \textbf{Median AUC} & \textbf{Mean F1} & \textbf{Mean Acc.} & \textbf{Mean Complexity} \\
\midrule
\multicolumn{7}{l}{\textit{Overall (40 datasets)}} \\
EBM & 0.9188 & 0.0987 & 0.9516 & 0.8162 & 0.8835 & 3{,}273 \\
LightGBM & 0.9174 & 0.1014 & 0.9553 & 0.8215 & 0.8840 & 3{,}415 \\
XGBoost & 0.9168 & 0.1004 & 0.9485 & 0.8161 & 0.8813 & 1{,}708 \\
RuleFit & 0.9118 & 0.1010 & 0.9419 & 0.7739 & 0.8748 & 131 \\
\textbf{\IAR{}} & \textbf{0.9105} & 0.0971 & 0.9335 & 0.8037 & 0.8747 & \textbf{122} \\
\textbf{\IAA{}} & \textbf{0.9048} & 0.1041 & 0.9328 & 0.7898 & 0.8688 & \textbf{120} \\
\midrule
\multicolumn{7}{l}{\textit{Real-world (24 datasets)}} \\
LightGBM & 0.9094 & 0.0860 & 0.9354 & 0.7796 & 0.8785 & 3{,}051 \\
EBM & 0.9080 & 0.0805 & 0.9258 & 0.7675 & 0.8741 & 2{,}278 \\
XGBoost & 0.9079 & 0.0836 & 0.9340 & 0.7710 & 0.8739 & 1{,}674 \\
RuleFit & 0.9003 & 0.0848 & 0.9244 & 0.7304 & 0.8660 & 133 \\
\textbf{\IAR{}} & \textbf{0.9002} & 0.0809 & 0.9131 & 0.7577 & 0.8681 & \textbf{140} \\
\textbf{\IAA{}} & \textbf{0.8980} & 0.0819 & 0.9115 & 0.7546 & 0.8665 & \textbf{134} \\
\midrule
\multicolumn{7}{l}{\textit{Synthetic (16 datasets)}} \\
EBM & 0.9350 & 0.1221 & 0.9747 & 0.8893 & 0.8975 & 4{,}765 \\
XGBoost & 0.9301 & 0.1233 & 0.9710 & 0.8837 & 0.8924 & 1{,}760 \\
LightGBM & 0.9295 & 0.1231 & 0.9703 & 0.8843 & 0.8923 & 3{,}962 \\
RuleFit & 0.9289 & 0.1223 & 0.9676 & 0.8392 & 0.8879 & 127 \\
\textbf{\IAR{}} & \textbf{0.9259} & 0.1187 & 0.9529 & 0.8725 & 0.8846 & \textbf{94} \\
\textbf{\IAA{}} & \textbf{0.9151} & 0.1330 & 0.9611 & 0.8427 & 0.8724 & \textbf{99} \\
\bottomrule
\end{tabular}%
}
\end{table}

\begin{figure}[t]
\centering
\includegraphics[width=0.7\textwidth]{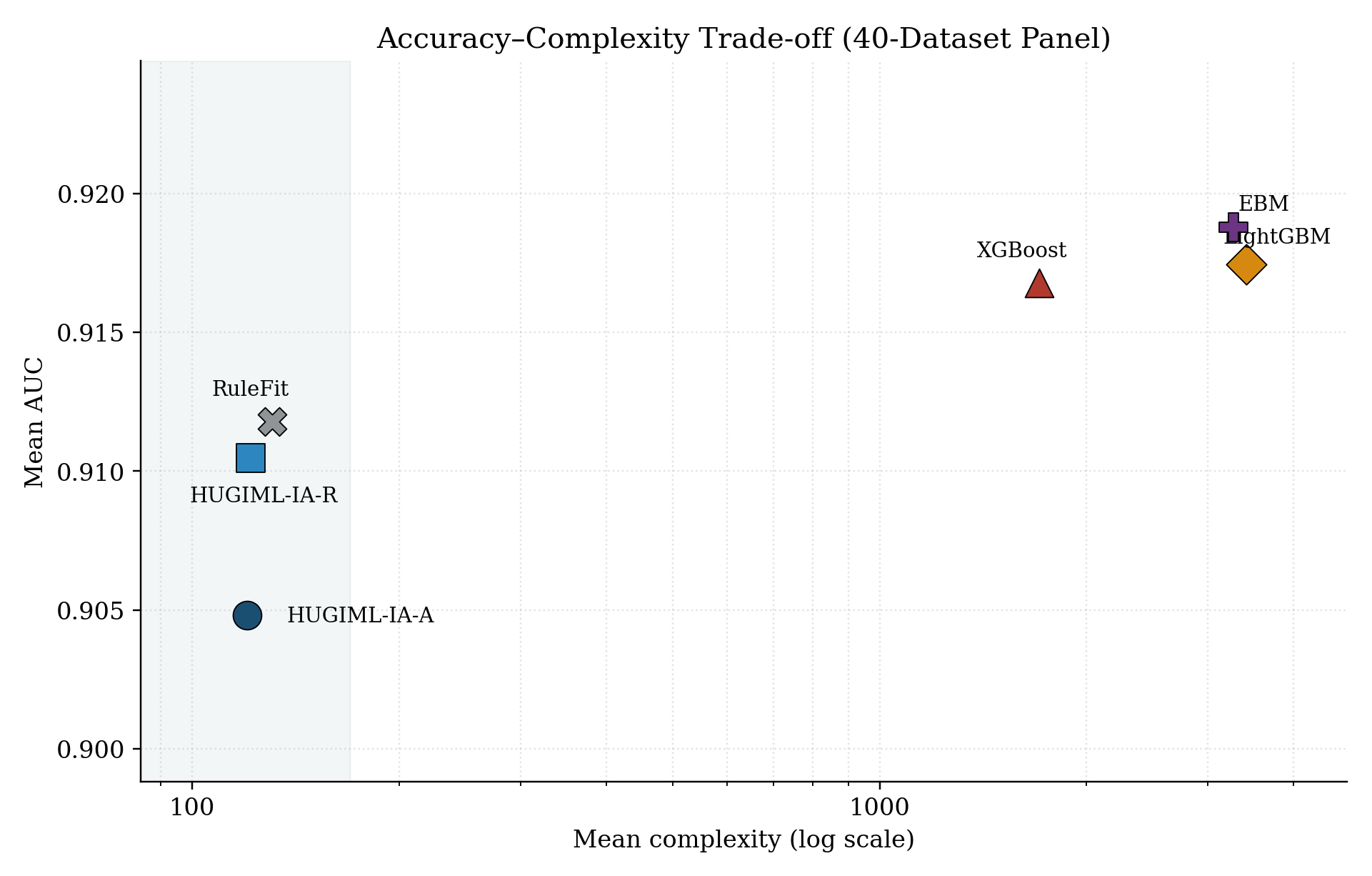}
\caption{Accuracy--complexity trade-off across the six evaluated model families (40-dataset panel). IAIML remains within 1.4 AUC points of the strong predictive baselines at 14--28$\times$ fewer fitted explanation units, while occupying the same compact-complexity tier as RuleFit.}
\label{fig:frontier}
\end{figure}

Table~\ref{tab:main_validation} and Figure~\ref{fig:frontier} summarize the headline result. EBM reaches the highest mean AUC (0.9188), followed by LightGBM (0.9174), XGBoost (0.9168), RuleFit (0.9118), \IAR{} (0.9105), and \IAA{} (0.9048): a span of 1.40 AUC points separates the best and worst mean performers on this panel. A Friedman test across all six models rejects the null hypothesis of equal performance ($\chi^2 = 30.56$, $p < 0.001$, $N=40$ datasets), with mean ranks of 2.75 (EBM), 2.80 (LightGBM), 3.20 (XGBoost), 3.85 (RuleFit), 3.85 (\IAA{}), and 4.55 (\IAR{}); lower rank indicates better relative performance.

This overall ordering is not stable across the two halves of the panel, as shown by the real-world and synthetic blocks in Table~\ref{tab:main_validation}. On the 24 real-world datasets, LightGBM narrowly leads (0.9094) over EBM and XGBoost, which are themselves separated by only 0.0001 AUC (0.9080 vs.\ 0.9079), and the three compact methods, RuleFit and both IAIML pathways, cluster within 0.0023 AUC of one another at the bottom of the ordering (0.8980--0.9003). On the 16 synthetic datasets the same six methods spread out considerably further: EBM opens a clearer lead (0.9350) over the tree ensembles (0.9301, 0.9295), while \IAA{} falls to 0.9151, 1.4 points behind RuleFit's 0.9289, a much wider gap than the even spacing seen on real-world data. This asymmetry follows directly from panel composition rather than from any real-world/synthetic distinction as such: the synthetic block was built to include stress tests on both sides of IAIML's pairwise, algebraic-operator design boundary, so it stretches the ordering in both directions at once, whereas the real-world block reflects a narrower band of naturally occurring, mostly diffuse dependency structure that gives no single interaction-handling mechanism as much room to separate from the rest. Complexity shows a parallel asymmetry: EBM's mean footprint nearly doubles from the real-world to the synthetic block (2{,}278 to 4{,}765 components) as its lookup tables expand to cover sharper synthetic decision boundaries, while both IAIML pathways and RuleFit stay within a narrow 94--140-component band in both blocks.

\begin{table}[t]
\centering
\caption{Holm-corrected pairwise Wilcoxon signed-rank tests, \IAA{} as reference, on per-dataset mean AUC ($N=40$).}
\label{tab:pairwise}
\small
\begin{tabular}{lrrrrrr}
\toprule
\textbf{Comparison} & \textbf{Mean $\Delta$AUC} & \textbf{Wins} & \textbf{Losses} & \textbf{Ties} & \textbf{Wilcoxon $p$} & \textbf{Holm $p$} \\
\midrule
vs.\ \IAR{} & $-0.0057$ & 17 & 14 & 9 & 0.652 & 0.652 \\
vs.\ XGBoost & $-0.0120$ & 16 & 23 & 1 & 0.072 & 0.216 \\
vs.\ LightGBM & $-0.0126$ & 13 & 26 & 1 & 0.033 & 0.131 \\
vs.\ EBM & $-0.0140$ & 9 & 29 & 2 & 0.001 & \textbf{0.006} \\
vs.\ RuleFit & $-0.0070$ & 17 & 22 & 1 & 0.302 & 0.604 \\
\bottomrule
\end{tabular}
\end{table}

Table~\ref{tab:pairwise} sharpens this picture. Under the reference-based Holm correction, \IAA{} shows no significant AUC difference from XGBoost ($p=0.216$), LightGBM ($p=0.131$), \IAR{} ($p=0.652$), or RuleFit ($p=0.604$), while EBM retains a significant edge ($p=0.006$). The main result is a compactness-versus-performance trade-off rather than an outright accuracy claim: EBM, LightGBM, and XGBoost have the highest mean AUCs, while the two IAIML variants trail by 0.6--1.4 AUC points on average, attaining this range with 14--28 times fewer fitted explanation components.

\subsubsection{Compactness and the RuleFit Comparison}

The compactness result is strongest relative to the two tree ensembles and EBM: \IAA{} uses 120 selected patterns/features on average, versus 1{,}708 realized leaves for XGBoost ($14.2\times$ more), 3{,}415 for LightGBM ($28.5\times$ more), and 3{,}273 lookup-table entries for EBM ($27.3\times$ more); \IAR{}, at 122 components, is similarly compact. Against RuleFit, whose fitted models retain 131 active rule and linear terms on average, the three compact families occupy the same tier: both IAIML variants and RuleFit lie in the 120--131 component range, and the pairwise AUC gap between \IAA{} and RuleFit is not statistically significant (Table~\ref{tab:pairwise}). Section~\ref{sec:comp_cost} shows that hyperparameter-tuning cost further separates IAIML from RuleFit.

This component-count parity compares counts of auditable units, not the complexity of each unit, and on the latter axis RuleFit's units are individually more complex. RuleFit's underlying trees are grown with \texttt{tree\_size}~$\in\{5,10\}$ (Appendix Table~\ref{tab:hyper_baseline_validation}), a parameter that fixes the number of terminal nodes per tree rather than a depth; reaching $M$ terminal nodes requires depth at least $\lceil\log_2 M\rceil$ in the balanced case, so trees with 5 or 10 leaves typically yield rules spanning three to four antecedent conditions. IAIML patterns are drawn from a grid capped at $L\leq 2$ for the main comparison (Appendix Table~\ref{tab:hyper_iaiml_validation}), and the inner-CV loop selects the longer setting, $L=2$, in only a slight majority of outer-training folds for both pathways (57\% for \IAA{}, 54\% for \IAR{}; Appendix Table~\ref{tab:hyperfreq}). A RuleFit rule therefore conjoins, on average, more original-feature conditions than an IAIML pattern, even though the two families report similar total component counts in Table~\ref{tab:main_validation}. The audit-load proxy used throughout this paper counts components rather than conditions per component, so it understates this difference: a reviewer auditing a RuleFit model is typically reading longer individual conditions than one auditing an IAIML model of comparable nominal size.

\subsection{Computational Cost}
\label{sec:comp_cost}

Compactness of the fitted model is only part of the practical cost of using an interpretable method; the other part is how expensive the model is to tune. Table~\ref{tab:runtime_main} reports the mean wall-clock cost of the full inner-CV hyperparameter search per outer fold, alongside the mean refit time of the selected configuration.

\begin{table}[t]
\centering
\caption{Mean fit and hyperparameter-tuning time per model, averaged over all outer folds and datasets, sorted by tuning time.}
\label{tab:runtime_main}\setlength{\tabcolsep}{5pt}
\small
\begin{tabular}{lrr}
\toprule
\textbf{Model} & \textbf{Mean fit time (s)} & \textbf{Mean tune time (s)} \\
\midrule
\IAR{} & 0.52 & 7.27 \\
XGBoost & 0.45 & 7.75 \\
LightGBM & 0.39 & 8.01 \\
\IAA{} & 1.31 & 16.56 \\
EBM & 2.26 & 21.10 \\
RuleFit & 15.38 & 126.41 \\
\bottomrule
\end{tabular}
\end{table}

\IAR{} is the fastest model to tune (7.3~seconds, on a log scale essentially tied with XGBoost's 7.8~seconds and LightGBM's 8.0~seconds); \IAA{}'s tuning time (16.6~seconds) is comparable to EBM's (21.1~seconds). RuleFit is a clear outlier: its mean tuning time, 126.4~seconds, is roughly $6\times$ EBM's and $17\times$ \IAR{}'s, and its mean fit time (15.4~seconds) is likewise an order of magnitude above every other model's (0.4--2.3~seconds). This cost stems from RuleFit's need to grow an auxiliary tree ensemble before rule extraction and lasso fitting, repeated across every hyperparameter configuration in the inner-CV grid. While RuleFit and IAIML occupy the same compact-complexity tier in Figure~\ref{fig:frontier}, they do not occupy the same practical-cost tier under the evaluated implementation and grids.

\subsection{Robustness Across Datasets}
\label{sec:robustness}

\begin{figure}[t]
\centering
\includegraphics[width=0.9\textwidth]{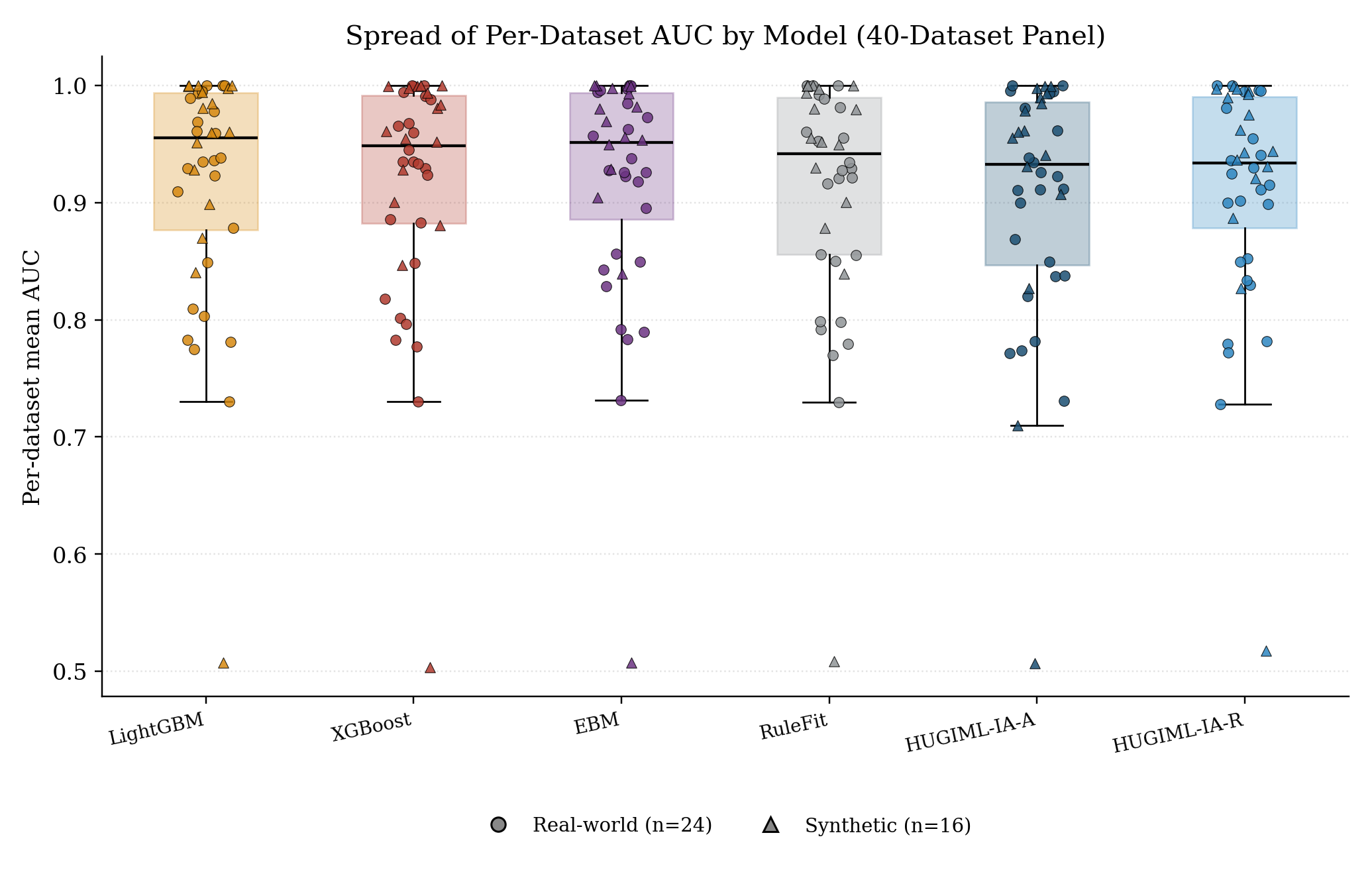}
\caption{Distribution of per-dataset mean AUC for each model across the 40-dataset panel (box plot with individual dataset points, real-world circles and synthetic triangles). The six interquartile ranges overlap substantially; differences between models are small relative to dataset-to-dataset variation.}
\label{fig:distribution}
\end{figure}

Table~\ref{tab:main_validation}'s Std.\ AUC column shows that the six models have similar cross-dataset standard deviations (0.097--0.104), and Figure~\ref{fig:distribution} confirms that their per-dataset AUC distributions overlap substantially. The low-AUC outliers, in the range 0.50--0.80, include both hard real-world datasets (\texttt{bank\_marketing}, \texttt{compas\_recidivism}, \texttt{default\_credit\_card}) and the deliberately intractable \texttt{SynthParityGroups}. The mean-AUC ordering in Table~\ref{tab:main_validation} is a real but fine-grained ranking superimposed on a much larger shared source of dataset-to-dataset variation.

\subsection{Performance by Dataset Group and Interaction Structure}
\label{sec:group}

\begin{figure}[t]
\centering
\includegraphics[width=0.75\textwidth]{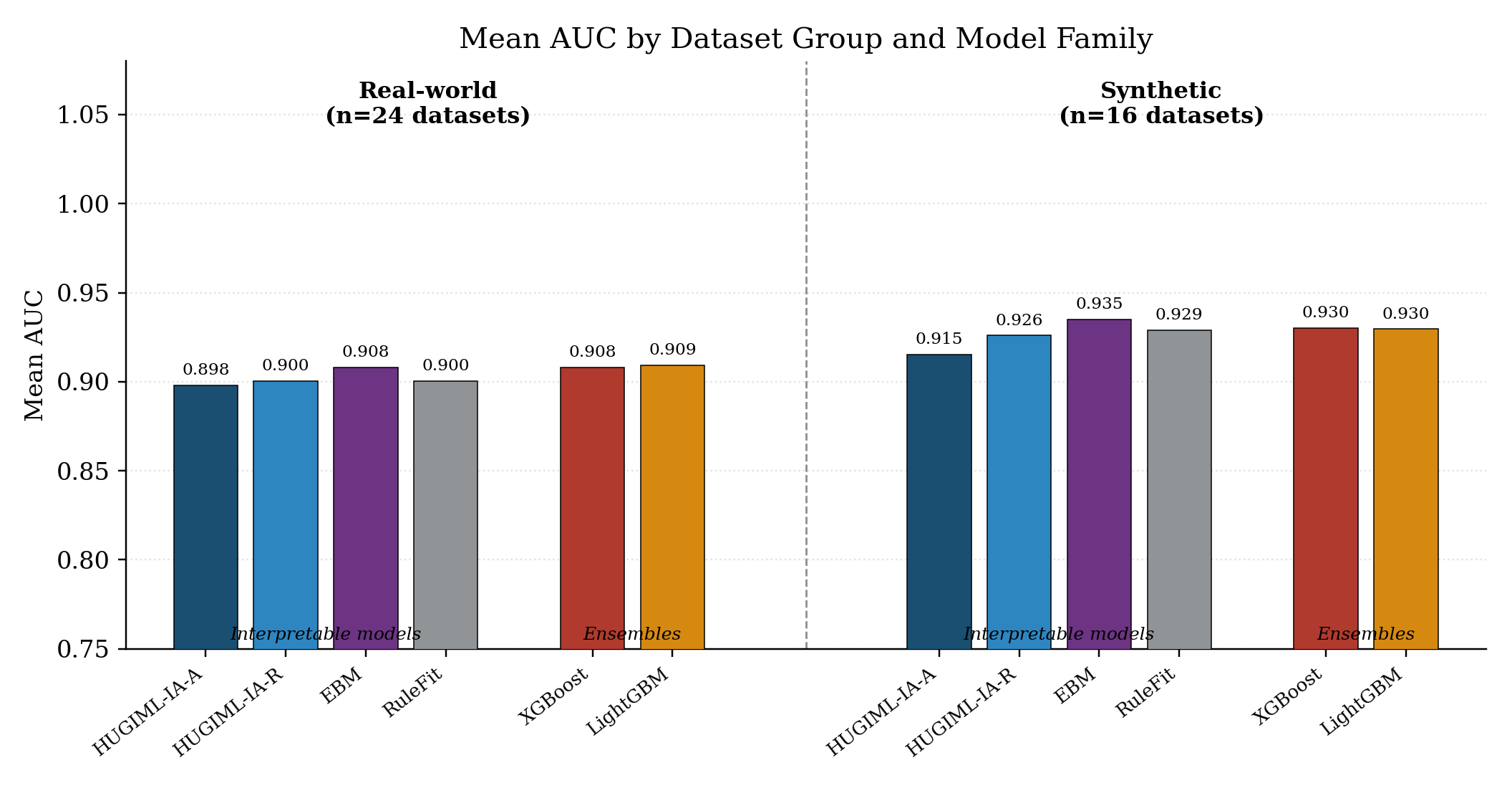}
\caption{Mean AUC by dataset group, with models clustered into two families for direct within-family comparison: interpretable methods (both IAIML pathways, EBM, RuleFit) and tree-ensemble baselines (XGBoost, LightGBM).}
\label{fig:group}
\end{figure}

Figure~\ref{fig:group} presents the same real-world/synthetic breakdown as Table~\ref{tab:main_validation}, but clusters models into the two families the comparison is organized around, so that alternate ensembles and alternate interpretable methods can each be read directly rather than interleaved by rank. Within the ensemble family, XGBoost and LightGBM are essentially interchangeable on both dataset groups (a 0.0015 AUC gap on real-world data, 0.0006 on synthetic), so the tree-ensemble baseline is not sensitive to which of the two implementations is chosen. Within the interpretable family the spread is wider and less stable: EBM leads and \IAA{} trails on both dataset groups, and the four-method spread widens from 0.0100 AUC on real-world data to 0.0199 on synthetic data, consistent with the operator-coverage limitation discussed next being more exposed on the synthetic panel.

The expanded 40-dataset panel includes several synthetic benchmarks that isolate specific interaction structures, enabling a fine-grained view of where IAIML's design provides distinctive value. The clearest gains appear on datasets built around low-marginal pairwise structure (refer to Appendix~\ref{app:perdatasetAUC}), the regime IAIML's interaction criterion directly targets. On \texttt{synth\_high\_dim\_buried\_pairs}, where two pairwise interactions are embedded among 100 features, \IAA{} reaches AUC 0.960, outperforming XGBoost (0.881), LightGBM (0.870), and RuleFit (0.878), and narrowly exceeding EBM (0.956), all at 100 components versus 2{,}971--8{,}886 for the baselines. On \texttt{synth\_multi\_pairwise}, where the label depends on three independent pairwise products with zero marginal signal, \IAA{} reaches 0.978 versus 0.970 for EBM and 0.961 for XGBoost, and on \texttt{synth\_correlated\_masked}, where the true interaction pair is obscured by correlated proxy features, \IAA{} again leads (0.907 vs.\ 0.904 for EBM). These results confirm that IAIML's marginal-signal-independent interaction scoring provides measurable gains precisely where its core design assumption, that important feature pairs carry negligible marginal signal, is satisfied. The same benefit extends to several small, interaction-rich real-world datasets: on \texttt{hepatitis} ($n=155$), \texttt{heart\_disease\_cleveland} ($n=303$), and \texttt{pima\_diabetes} ($n=768$), the best IAIML variant exceeds all baselines by a modest but consistent 0.5--1.0 AUC points, helped both by the compact model's lower variance at small sample sizes and by interaction-aware admission recovering clinically plausible feature combinations such as bilirubin--albumin, age--cholesterol, and glucose--insulin pairs.

The panel also contains two structures that fall outside this design envelope. \texttt{SynthParityGroups} encodes a parity function across two groups of three features, a genuinely higher-order interaction that no pairwise criterion can recover; all six models perform near chance on it (AUC 0.50--0.52). \texttt{SynthModularPairwise} is more diagnostic: \IAA{} drops to AUC 0.710 while \IAR{} recovers to 0.997, a divergence that traces to a specific failure of the algebraic pair-transform family (product, difference, sum) on a modular-arithmetic boundary that none of these four operators linearizes. Because the relaxed-mining pathway preserves the full discretized pattern vocabulary rather than committing to named algebraic operators, it recovers the signal that \IAA{}'s operators miss on this dataset.

\subsection{Ablation Studies}
\label{sec:ablation}

\begin{figure}[b]
\centering
\includegraphics[width=0.9\textwidth]{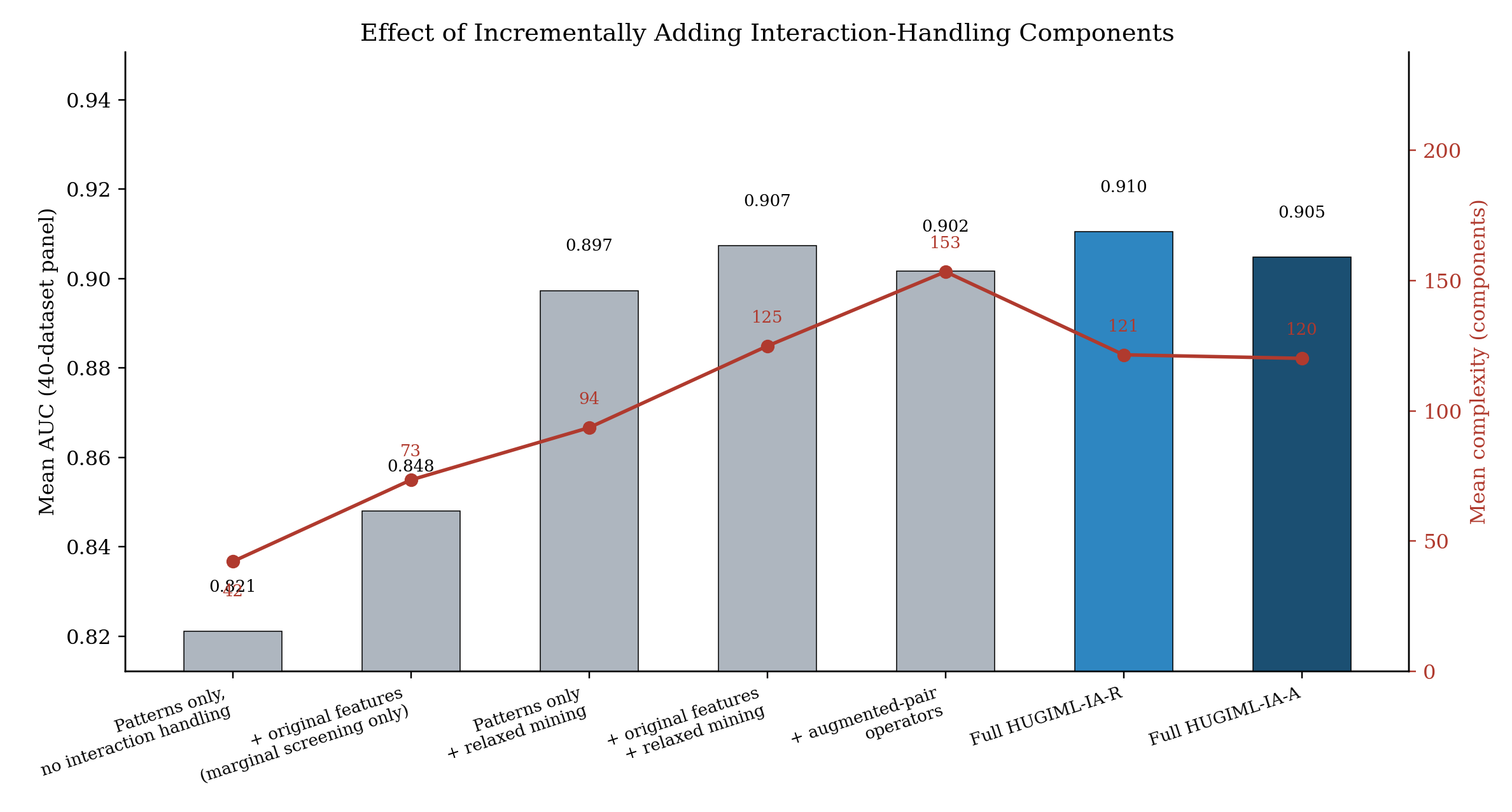}
\caption{Incremental effect of adding interaction-handling components (40-dataset panel).}
\label{fig:ablation}
\end{figure}

To isolate the contribution of each design element, we evaluate 19 component-isolated configurations on the full 40-dataset panel (Table~\ref{tab:ablation_full} in Appendix~\ref{app:ablation_full}). Figure~\ref{fig:ablation} traces the subset of this progression obtained by adding one interaction-handling mechanism at a time: a patterns-only configuration with no interaction handling reaches mean AUC 0.821 at 42 components; admitting original features under marginal screening alone lifts this to 0.848; adding relaxed mining to the patterns-only configuration reaches 0.897; combining original features with relaxed mining reaches 0.907; and adding augmented-pair operators instead reaches 0.902. The two rightmost bars show the fully tuned models from the main comparison, \IAR{} (0.9105) and \IAA{} (0.9048), each slightly above its matched ablation-harness configuration because the main comparison additionally tunes pattern length and the information-gain threshold via nested cross-validation (Appendix~\ref{app:hyperparams}) rather than fixing them at the ablation default. The net effect of the full interaction-handling framework, relative to the patterns-only baseline, is a gain of approximately 8.6 AUC points.

The remaining 14 configurations in Table~\ref{tab:ablation_full}, not shown in Figure~\ref{fig:ablation}, isolate two further points. First, extending relaxed mining to length-3 patterns (\texttt{RelaxedMining\_L3\_Survivors20}) reduces mean AUC to 0.849, well below the matched length-2 variant's 0.907, confirming that deeper pattern spaces dilute rather than strengthen the mined evidence. Second, replacing the interaction-information admission criterion with a marginal information-gain criterion (\texttt{AugPair\_L2\_MarginalIG\_Source20}) drops mean AUC to 0.853 from 0.902 for the matched interaction-information variant, a direct confirmation that the synergy-based selection criterion carries predictive value beyond what marginal screening alone recover.

\FloatBarrier
\section{Discussion}
\label{sec:discussion}

\citet{Rudin2019} has argued that the perceived trade-off between interpretability and accuracy is often unnecessary and that interpretable models can match opaque ones in many high-stakes settings. The results above provide additional evidence from the interaction-aware regime. IAIML does not outperform the strongest baselines in mean AUC across the full panel: EBM holds a significant edge and the tree ensembles hold non-significant but consistent advantages. Its value lies instead in the compactness of its explanation, closing much of the accuracy gap while using 14--28 times fewer fitted components than EBM, XGBoost, or LightGBM. RuleFit reaches a similar level of compactness with a small AUC edge over IAIML; the two families occupy the same compact-complexity tier but differ in explanation semantics, in typical rule or pattern complexity, and in tuning cost, as detailed in Section~\ref{sec:main_results}. IAIML's distinctive contribution within this tier is the explicit interaction-screening criterion and the ability to route interaction evidence into either a pure pattern vocabulary or a sparse algebraic-pair vocabulary, a choice the other compact methods do not offer. The clearest evidence for this design rationale comes from the datasets where pairwise interactions with low marginal signal carry the label signal: as observed in our experimental results, IAIML outperforms every baseline, including EBM, on exactly this subset, confirming that the marginal-signal-independent synergy criterion addresses a real gap in existing methods even though the gap does not dominate the average across a broad benchmark panel.

The three interaction-aware families evaluated here, IAIML, EBM, and RuleFit, locate interaction structure through different mechanisms, and the comparison is informative about each mechanism's blind spots rather than about a single winner. EBM's FAST algorithm \citep{Lou2013} ranks candidate pairs by residual improvement after main-effect fitting and attains the best mean rank on this panel (2.75) at a cost of 3{,}273 lookup-table entries. RuleFit captures interaction structure implicitly through tree splits and reaches mean AUC 0.9118 at 131 active terms with no significant AUC difference from \IAA{} (Holm $p=0.604$); as the compactness comparison notes, however, RuleFit's rules are individually more complex than IAIML's length-capped patterns, so the similar component counts understate IAIML's audit-load advantage. The ablation study isolates the value of IAIML's own criterion directly: replacing interaction-information admission with marginal information-gain admission costs approximately 5 AUC points, confirming that the synergy-based selection criterion carries value beyond marginal screening. Between IAIML's two routing pathways, \IAR{} slightly outperforms \IAA{} overall (0.9105 vs.\ 0.9048) with no significant difference between them (Holm $p=0.652$), a gap driven in part by \IAA{}'s algebraic operators failing on non-linearizable structures such as the modular-arithmetic boundary in \texttt{SynthModularPairwise}. Practitioners can therefore select the pathway that matches their governance requirements, rule-like patterns versus named algebraic terms, without expecting a systematic accuracy trade-off.

The complexity budget also provides a natural interface for model-risk governance: organizations can set $\Bmax$ to match audit capacity, ensuring that every retained pattern, feature, or pair term receives human review. The compactness advantage over tree ensembles and EBM, 14--28 times fewer fitted units under the audit-load proxy, is most relevant where regulatory scrutiny of individual predictions is common. Against RuleFit the comparison does not reduce to a size argument: \IAA{}, \IAR{}, and RuleFit (120, 122, and 131 mean components) occupy the same compact-complexity tier with no significant AUC differences among the three. The more informative differentiator is tuning economics: \IAR{}'s tuning time (7.3~seconds) is the fastest in the panel, roughly 17 times cheaper than RuleFit's (126.4~seconds), a gap that recurs every time the model is retuned. Practitioners choosing among the three compact interpretable families are therefore choosing primarily on governance fit and tuning economics rather than trading final model size against accuracy. Exhaustive pairwise interaction scoring scales as $O(p^2n)$; the partner-budget variant reduces this to $O(npr)$ and remains appropriate for the moderate-dimensional regulated tabular datasets that motivate this work.

IAIML is deliberately scoped to pairwise interaction structure: the interaction criterion, the guarantees in Section~\ref{sec:theory}, and the four augmented-pair operators are all built around two-way synergy, and extending the criterion to order-$k$ interactions would introduce $O(p^k)$ combinatorial cost together with unresolved definitional questions for multi-variable partial information decomposition \citep{Williams2010}. This boundary is visible empirically wherever the panel requires structure outside it. On \texttt{SynthParityGroups}, whose label depends on an inherently higher-order parity function, every model in the comparison performs near chance, not IAIML alone. On \texttt{SynthXORInteractions} and \texttt{SynthModularPairwise}, the four fixed algebraic operators, product, absolute difference, signed difference, and sum, do not linearize the underlying interaction: \IAA{} and \IAR{} reach only 0.941 and 0.943 AUC on the former, slightly behind all four baselines (0.949--0.952), and \IAA{} collapses to 0.710 on the latter, though \IAR{}'s discretized pattern vocabulary recovers most of the signal the algebraic operators miss on both. A richer or learnable operator family is the natural way to narrow this specific gap without abandoning the pairwise scope.

Both pattern mining and interaction estimation operate on discretized features; adaptive binning mitigates but does not eliminate the resulting information loss, and continuous mutual-information estimators \citep{Kraskov2004} could reduce it further at higher computational cost. 
\FloatBarrier
\section{Conclusion}
\label{sec:conclusion}

This paper introduced IAIML, an interaction-aware pattern classification for tabular binary classification. The framework addresses a specific limitation of marginal pattern-screening frameworks by scoring coarse-grid pairwise interaction evidence before final candidate selection, then routing that evidence through either interaction-relaxed mining (\IAR{}) or augmented-pair operators (\IAA{}). A partitioned component budget keeps the final sparse classifier within a bounded explanation footprint.

On a 40-dataset panel (24 real-world, 16 synthetic), IAIML achieves mean AUC within 1.4 points of the strongest baselines while using 14--28 times fewer fitted explanation components. On datasets with strong pairwise interaction structure and low marginal signal, IAIML outperforms all baselines, confirming that the marginal-signal-independent synergy criterion addresses a genuine gap. Performance degrades on datasets requiring higher-order interactions or rich diffuse nonlinearity, delineating the method's scope. These findings support interaction-aware pattern classification as a practical option when bounded explanation size and controlled treatment of feature interactions are central requirements.

\bibliographystyle{apalike}
\bibliography{references}

\FloatBarrier
\begin{appendix}

\section{Notation Table}
\label{app:notation}

\begin{table}[htbp]
\centering
\caption{Summary of principal notation used.}
\label{tab:notation}
\footnotesize
\begin{tabular}{ll}
\toprule
\textbf{Symbol} & \textbf{Description} \\
\midrule
$n$ & Number of instances (rows) in the dataset \\
$p$ & Number of predictors (features) \\
$\mathcal{D}$, $\Dtr$, $\Dtst$ & Full, training, and test datasets \\
$f_j$ & The $j$-th feature ($j = 1,\ldots,p$) \\
$y$ & Binary target label $\in \{0,1\}$ \\
$B$, $B_j$, $B_j^*$ & Global, per-feature, and selected bin count \\
$\IG(f_j)$ & Information gain of discretized feature $f_j$ w.r.t.\ $y$ \\
$\II(f_i, f_j)$ & Interaction information of feature pair $(f_i, f_j)$ \\
$I(X;Y)$ & Mutual information between $X$ and $Y$ \\
$\EU(f_j,c)$ & External utility of feature $f_j$ for class $c$ \\
$\Phi_{ij}$ & Pair transform family for features $(f_i, f_j)$ \\
$\mathcal{S}$ & Set of selected source--partner pairs \\
$\Bmax$ & Global complexity budget (max features in classifier) \\
$B_p, B_o, B_a$ & Budget allocations: patterns, originals, augmented \\
$K$ & Number of target classes (typically $K=2$) \\
$L$ & Maximum pattern length \\
$G$ & Information gain threshold for HUG patterns \\
$\rho$ & Elbow-stopping threshold for adaptive binning \\
$\tau$ & Interaction information threshold for source selection \\
$\alpha$ & Coarsen-by-default fraction for pair-aware binning \\
$R(h)$, $\hat{R}(h)$ & True risk and empirical risk of classifier $h$ \\
$d_{\VC}$ & VC-dimension \\
\IAA{}, \IAR{} & Augmented-pair and interaction-relaxed-mining IAIML variants \\
\bottomrule
\end{tabular}
\end{table}

\newpage 
\section{Dataset Panel}
\label{app:datasets}

Table~\ref{tab:validation_datasets} lists the dataset panel used in the empirical study: 24 real-world datasets spanning credit risk, marketing response, health, and pricing domains, and 16 synthetic interaction stress tests. 

\begin{table}[htbp]
\centering
\caption{Dataset panel used for the nested cross-validation empirical results. \textit{Pos.\ Rate} is the fraction of positive-class labels.}
\label{tab:validation_datasets}
\footnotesize
\begin{tabular}{rllrrr}
\toprule
\# & \textbf{Dataset} & \textbf{Group} & \textbf{Rows} & \textbf{Features} & \textbf{Pos.\ Rate} \\
\midrule
1 & adult\_income & Real-world & 48{,}842 & 14 & 0.239 \\
2 & australian\_credit & Real-world & 690 & 14 & 0.445 \\
3 & bank\_marketing & Real-world & 45{,}211 & 15 & 0.117 \\
4 & breast\_cancer\_wisconsin & Real-world & 569 & 30 & 0.373 \\
5 & california\_housing\_binarized & Real-world & 20{,}640 & 8 & 0.500 \\
6 & chess\_krvskp & Real-world & 3{,}196 & 36 & 0.522 \\
7 & compas\_recidivism & Real-world & 7{,}214 & 7 & 0.451 \\
8 & default\_credit\_card & Real-world & 30{,}000 & 23 & 0.221 \\
9 & fico\_heloc & Real-world & 9{,}871 & 23 & 0.480 \\
10 & german\_credit & Real-world & 1{,}000 & 20 & 0.300 \\
11 & heart\_disease\_cleveland & Real-world & 303 & 13 & 0.459 \\
12 & hepatitis & Real-world & 155 & 19 & 0.794 \\
13 & ionosphere & Real-world & 351 & 33 & 0.641 \\
14 & magic\_gamma & Real-world & 19{,}020 & 10 & 0.352 \\
15 & mushroom & Real-world & 8{,}124 & 21 & 0.482 \\
16 & nursery & Real-world & 12{,}960 & 8 & 0.667 \\
17 & online\_shoppers\_intention & Real-world & 12{,}330 & 17 & 0.155 \\
18 & phoneme & Real-world & 5{,}404 & 5 & 0.293 \\
19 & pima\_diabetes & Real-world & 768 & 8 & 0.349 \\
20 & spambase & Real-world & 4{,}601 & 57 & 0.394 \\
21 & telco\_churn & Real-world & 7{,}043 & 19 & 0.265 \\
22 & vehicle\_silhouettes & Real-world & 846 & 18 & 0.235 \\
23 & waveform\_v2 & Real-world & 5{,}000 & 40 & 0.338 \\
24 & wine\_quality\_binarized & Real-world & 6{,}497 & 11 & 0.197 \\
\midrule
25 & SynthAdditiveNonlinear & Synthetic & 10{,}000 & 10 & 0.500 \\
26 & SynthCategoricalRules & Synthetic & 10{,}000 & 12 & 0.349 \\
27 & SynthCirclesNonlinear & Synthetic & 10{,}000 & 10 & 0.500 \\
28 & SynthImbalancedRare & Synthetic & 10{,}000 & 10 & 0.123 \\
29 & SynthLinearLowDim & Synthetic & 10{,}000 & 5 & 0.498 \\
30 & SynthMixedMissing & Synthetic & 10{,}000 & 10 & 0.497 \\
31 & SynthModularPairwise & Synthetic & 9{,}999 & 10 & 0.333 \\
32 & SynthMoonsNonlinear & Synthetic & 10{,}000 & 10 & 0.500 \\
33 & SynthParityGroups & Synthetic & 10{,}000 & 12 & 0.504 \\
34 & SynthSparseWide & Synthetic & 10{,}000 & 50 & 0.501 \\
35 & SynthXORInteractions & Synthetic & 10{,}000 & 10 & 0.503 \\
36 & synth\_cat\_continuous\_interact & Synthetic & 10{,}000 & 10 & 0.193 \\
37 & synth\_correlated\_masked & Synthetic & 10{,}000 & 15 & 0.497 \\
38 & synth\_high\_dim\_buried\_pairs & Synthetic & 10{,}000 & 100 & 0.495 \\
39 & synth\_multi\_pairwise & Synthetic & 10{,}000 & 12 & 0.502 \\
40 & synth\_threshold\_steps & Synthetic & 10{,}000 & 8 & 0.467 \\
\bottomrule
\end{tabular}
\end{table}

\newpage 
\section{Per-Dataset AUC Results}
\label{app:perdatasetAUC}

\begin{table}[htbp]
	\centering
	\caption{Per-dataset mean AUC, nested cross-validation (40-dataset panel). Bold indicates the best AUC for each dataset (ties within 0.001 all bolded). IA-A/IA-R = IAIML-A/IAIML-R.}
	\label{tab:per_dataset_auc}
	\footnotesize
	\resizebox{\textwidth}{!}{\begin{tabular}{llrrrrrr}
			\toprule
			\textbf{Dataset} & \textbf{Grp} & \textbf{IA-A} & \textbf{IA-R} & \textbf{XGBoost} & \textbf{LightGBM} & \textbf{EBM} & \textbf{RuleFit} \\
			\midrule
			adult\_income & RW & 0.911 & 0.911 & \textbf{0.929} & \textbf{0.930} & 0.927 & 0.921 \\
			australian\_credit & RW & 0.935 & 0.930 & 0.935 & 0.936 & \textbf{0.938} & 0.930 \\
			bank\_marketing & RW & 0.774 & 0.780 & 0.802 & \textbf{0.803} & 0.792 & 0.792 \\
			breast\_cancer\_wisconsin & RW & \textbf{0.995} & \textbf{0.996} & 0.991 & 0.993 & 0.994 & 0.992 \\
			california\_housing\_binarized & RW & 0.938 & 0.941 & 0.965 & \textbf{0.969} & 0.957 & 0.952 \\
			chess\_krvskp & RW & 0.995 & 0.995 & \textbf{1.000} & \textbf{1.000} & 0.997 & \textbf{1.000} \\
			compas\_recidivism & RW & \textbf{0.731} & 0.728 & \textbf{0.730} & 0.730 & \textbf{0.731} & 0.730 \\
			default\_credit\_card & RW & 0.771 & 0.772 & \textbf{0.783} & \textbf{0.783} & \textbf{0.783} & 0.779 \\
			fico\_heloc & RW & 0.869 & 0.899 & \textbf{0.924} & \textbf{0.923} & 0.922 & 0.916 \\
			german\_credit & RW & 0.781 & 0.781 & 0.777 & 0.775 & \textbf{0.789} & 0.770 \\
			heart\_disease\_cleveland & RW & \textbf{0.900} & \textbf{0.900} & 0.883 & 0.879 & 0.895 & 0.855 \\
			hepatitis & RW & 0.820 & \textbf{0.852} & 0.796 & 0.781 & 0.843 & 0.798 \\
			ionosphere & RW & 0.926 & 0.936 & 0.968 & \textbf{0.978} & 0.973 & 0.955 \\
			magic\_gamma & RW & 0.912 & 0.902 & 0.935 & \textbf{0.938} & 0.918 & 0.921 \\
			mushroom & RW & \textbf{1.000} & \textbf{1.000} & \textbf{1.000} & \textbf{1.000} & \textbf{1.000} & \textbf{1.000} \\
			nursery & RW & \textbf{1.000} & \textbf{1.000} & \textbf{1.000} & \textbf{1.000} & \textbf{1.000} & \textbf{1.000} \\
			online\_shoppers\_intention & RW & 0.922 & 0.925 & 0.933 & \textbf{0.935} & 0.926 & 0.928 \\
			phoneme & RW & 0.911 & 0.915 & 0.945 & \textbf{0.961} & 0.926 & 0.934 \\
			pima\_diabetes & RW & \textbf{0.837} & 0.830 & 0.818 & 0.810 & 0.829 & 0.798 \\
			spambase & RW & 0.981 & 0.981 & 0.988 & \textbf{0.989} & 0.985 & 0.981 \\
			telco\_churn & RW & \textbf{0.849} & \textbf{0.849} & 0.848 & \textbf{0.849} & \textbf{0.850} & \textbf{0.850} \\
			vehicle\_silhouettes & RW & 0.993 & \textbf{0.995} & 0.994 & \textbf{0.995} & \textbf{0.996} & 0.989 \\
			waveform\_v2 & RW & 0.962 & 0.955 & 0.960 & 0.959 & \textbf{0.963} & 0.960 \\
			wine\_quality\_binarized & RW & 0.838 & 0.833 & 0.885 & \textbf{0.910} & 0.857 & 0.856 \\
			SynthAdditiveNonlinear & Syn & 0.962 & 0.962 & 0.981 & 0.981 & \textbf{0.982} & 0.980 \\
			SynthCategoricalRules & Syn & \textbf{1.000} & \textbf{0.999} & \textbf{1.000} & \textbf{1.000} & \textbf{1.000} & \textbf{1.000} \\
			SynthCirclesNonlinear & Syn & \textbf{1.000} & 0.996 & \textbf{1.000} & \textbf{1.000} & \textbf{1.000} & \textbf{0.999} \\
			SynthImbalancedRare & Syn & 0.955 & 0.944 & 0.955 & \textbf{0.960} & 0.954 & 0.952 \\
			SynthLinearLowDim & Syn & 0.993 & 0.993 & \textbf{0.994} & \textbf{0.994} & \textbf{0.993} & \textbf{0.994} \\
			SynthMixedMissing & Syn & \textbf{0.931} & \textbf{0.931} & 0.928 & 0.928 & 0.929 & 0.930 \\
			SynthModularPairwise & Syn & 0.710 & 0.997 & \textbf{0.999} & \textbf{1.000} & \textbf{1.000} & \textbf{1.000} \\
			SynthMoonsNonlinear & Syn & 0.990 & 0.990 & \textbf{0.998} & \textbf{0.998} & \textbf{0.998} & \textbf{0.997} \\
			SynthParityGroups & Syn & 0.507 & \textbf{0.518} & 0.503 & 0.507 & 0.507 & 0.508 \\
			SynthSparseWide & Syn & \textbf{0.985} & 0.975 & \textbf{0.984} & \textbf{0.985} & 0.980 & 0.980 \\
			SynthXORInteractions & Syn & 0.941 & 0.943 & \textbf{0.952} & \textbf{0.951} & 0.949 & 0.949 \\
			synth\_cat\_continuous\_interact & Syn & 0.998 & 0.997 & \textbf{1.000} & \textbf{1.000} & \textbf{1.000} & \textbf{1.000} \\
			synth\_correlated\_masked & Syn & \textbf{0.907} & 0.887 & 0.900 & 0.899 & 0.904 & 0.900 \\
			synth\_high\_dim\_buried\_pairs & Syn & \textbf{0.960} & 0.921 & 0.881 & 0.870 & 0.956 & 0.878 \\
			synth\_multi\_pairwise & Syn & \textbf{0.978} & 0.936 & 0.961 & 0.960 & 0.969 & 0.955 \\
			synth\_threshold\_steps & Syn & 0.827 & 0.827 & \textbf{0.847} & 0.840 & 0.839 & 0.840 \\
			\bottomrule
	\end{tabular}}
\end{table}

\newpage 
\section{Pairwise Model Comparisons}
\label{app:additional}

This appendix reports analyses that extend beyond the primary reference-based comparison in Section~\ref{sec:experiments}.

\subsection{Full Round-Robin Pairwise Comparison}
\label{app:roundrobin}

Table~\ref{tab:pairwise} in the main text follows the paper's reference-based design of testing every baseline against \IAA{} as a fixed reference, with Holm correction across those five comparisons. Table~\ref{tab:roundrobin} instead runs Wilcoxon signed-rank tests for \emph{all} $\binom{6}{2}=15$ pairs and applies Holm correction across all 15, a more conservative and more complete view of the statistical structure. Five of the fifteen comparisons reach significance under this stricter correction. Both IAIML pathways remain significantly behind EBM even under the joint correction (\IAA{} vs.\ EBM, $p=0.014$; \IAR{} vs.\ EBM, $p<0.001$), confirming that this gap is not an artifact of the reference-based comparison's narrower correction. RuleFit is significantly behind all three higher-mean-AUC baselines, XGBoost ($p=0.002$), LightGBM ($p=0.009$), and EBM ($p=0.033$), while showing no significant difference from either IAIML variant; together these place RuleFit at the bottom of the five-baseline ordering despite its compact rule count. No other pairwise comparison, including both comparisons among the tree ensembles and the comparison between the two IAIML pathways, reaches significance.

\begin{table}[htbp]
\centering
\caption{Full round-robin Wilcoxon signed-rank comparison across all $\binom{6}{2}=15$ model pairs, per-dataset mean AUC ($N=40$), Holm-corrected across all 15 comparisons jointly.}
\label{tab:roundrobin}
\footnotesize
\begin{tabular}{llrrrrrrl}
\toprule
\textbf{A} & \textbf{B} & \textbf{Mean $\Delta$(A$-$B)} & \textbf{Wins} & \textbf{Losses} & \textbf{Ties} & \textbf{Wilcoxon $p$} & \textbf{Holm $p$} & \textbf{Sig.} \\
\midrule
IAIML-A & IAIML-R & $-0.0057$ & 17 & 14 & 9 & 0.652 & 1.000 &  \\
IAIML-A & XGBoost & $-0.0120$ & 16 & 23 & 1 & 0.072 & 0.431 &  \\
IAIML-A & LightGBM & $-0.0126$ & 13 & 26 & 1 & 0.033 & 0.262 &  \\
IAIML-A & EBM & $-0.0140$ & 9 & 29 & 2 & 0.001 & 0.014 & \checkmark \\
IAIML-A & RuleFit & $-0.0070$ & 17 & 22 & 1 & 0.302 & 1.000 &  \\
IAIML-R & XGBoost & $-0.0063$ & 11 & 28 & 1 & 0.006 & 0.053 &  \\
IAIML-R & LightGBM & $-0.0070$ & 11 & 28 & 1 & 0.005 & 0.053 &  \\
IAIML-R & EBM & $-0.0083$ & 6 & 32 & 2 & 0.000 & $<$0.001 & \checkmark \\
IAIML-R & RuleFit & $-0.0013$ & 11 & 28 & 1 & 0.033 & 0.262 &  \\
XGBoost & LightGBM & $-0.0007$ & 15 & 24 & 1 & 0.241 & 1.000 &  \\
XGBoost & EBM & $-0.0020$ & 16 & 23 & 1 & 0.606 & 1.000 &  \\
XGBoost & RuleFit & $+0.0050$ & 27 & 12 & 1 & 0.000 & 0.002 & \checkmark \\
LightGBM & EBM & $-0.0013$ & 19 & 20 & 1 & 0.780 & 1.000 &  \\
LightGBM & RuleFit & $+0.0057$ & 28 & 10 & 2 & 0.001 & 0.009 & \checkmark \\
EBM & RuleFit & $+0.0070$ & 25 & 14 & 1 & 0.003 & 0.033 & \checkmark \\
\bottomrule
\end{tabular}
\end{table}

\subsection{Dataset-Group-Conditional Comparison}
\label{app:groupstats}

Table~\ref{tab:group_pairwise} repeats the reference-based test of Table~\ref{tab:pairwise} separately within the 24 real-world and 16 synthetic datasets, referenced from Section~\ref{sec:group}. On the real-world subset, \IAA{}'s gap to EBM reaches significance (Holm $p=0.005$); no other real-world comparison does. On the synthetic subset, none of the five comparisons reach significance, even though the raw mean AUC gaps are numerically larger there than on the real-world subset (e.g., $-0.0199$ vs.\ EBM, compared with $-0.0100$ on real-world data); this reflects the higher cross-dataset variance of the synthetic panel (Table~\ref{tab:main_validation}, Std.\ AUC column) rather than a smaller effect. The synthetic-subset results should be read as suggestive rather than conclusive; the real-world result corroborates the full-panel finding in Table~\ref{tab:pairwise}.

\begin{table}[htbp]
\centering
\caption{Reference-based Wilcoxon tests (\IAA{} vs.\ each baseline), computed separately within the real-world ($N=24$) and synthetic ($N=16$) dataset groups.}
\label{tab:group_pairwise}
\footnotesize
\begin{tabular}{lrrrrr}
\toprule
& \multicolumn{2}{c}{\textbf{Real-world ($N=24$)}} & & \multicolumn{2}{c}{\textbf{Synthetic ($N=16$)}} \\
\cmidrule(lr){2-3} \cmidrule(lr){5-6}
\textbf{Comparison} & \textbf{Mean $\Delta$AUC} & \textbf{Holm $p$} & & \textbf{Mean $\Delta$AUC} & \textbf{Holm $p$} \\
\midrule
vs.\ IAIML-R & $-0.0023$ & 0.723 & & $-0.0108$ & 0.789 \\
vs.\ XGBoost & $-0.0099$ & 0.221 & & $-0.0150$ & 1.000 \\
vs.\ LightGBM & $-0.0114$ & 0.221 & & $-0.0144$ & 1.000 \\
vs.\ EBM & $-0.0100$ & 0.005 & & $-0.0199$ & 1.000 \\
vs.\ RuleFit & $-0.0024$ & 0.723 & & $-0.0138$ & 1.000 \\
\bottomrule
\end{tabular}
\end{table}

\subsection{Per-Dataset Rank Distribution}
\label{app:rankdist}

Table~\ref{tab:main_validation} and Table~\ref{tab:pairwise} summarize model quality through panel-wide means and reference-based tests. Figure~\ref{fig:rank_heatmap} and Table~\ref{tab:rank_distribution} complement this with the full per-dataset AUC rank structure (1 = best AUC on that dataset), extending the mean-rank summary of Section~\ref{sec:main_results}.

\begin{figure}[htbp]
\centering
\includegraphics[width=\textwidth]{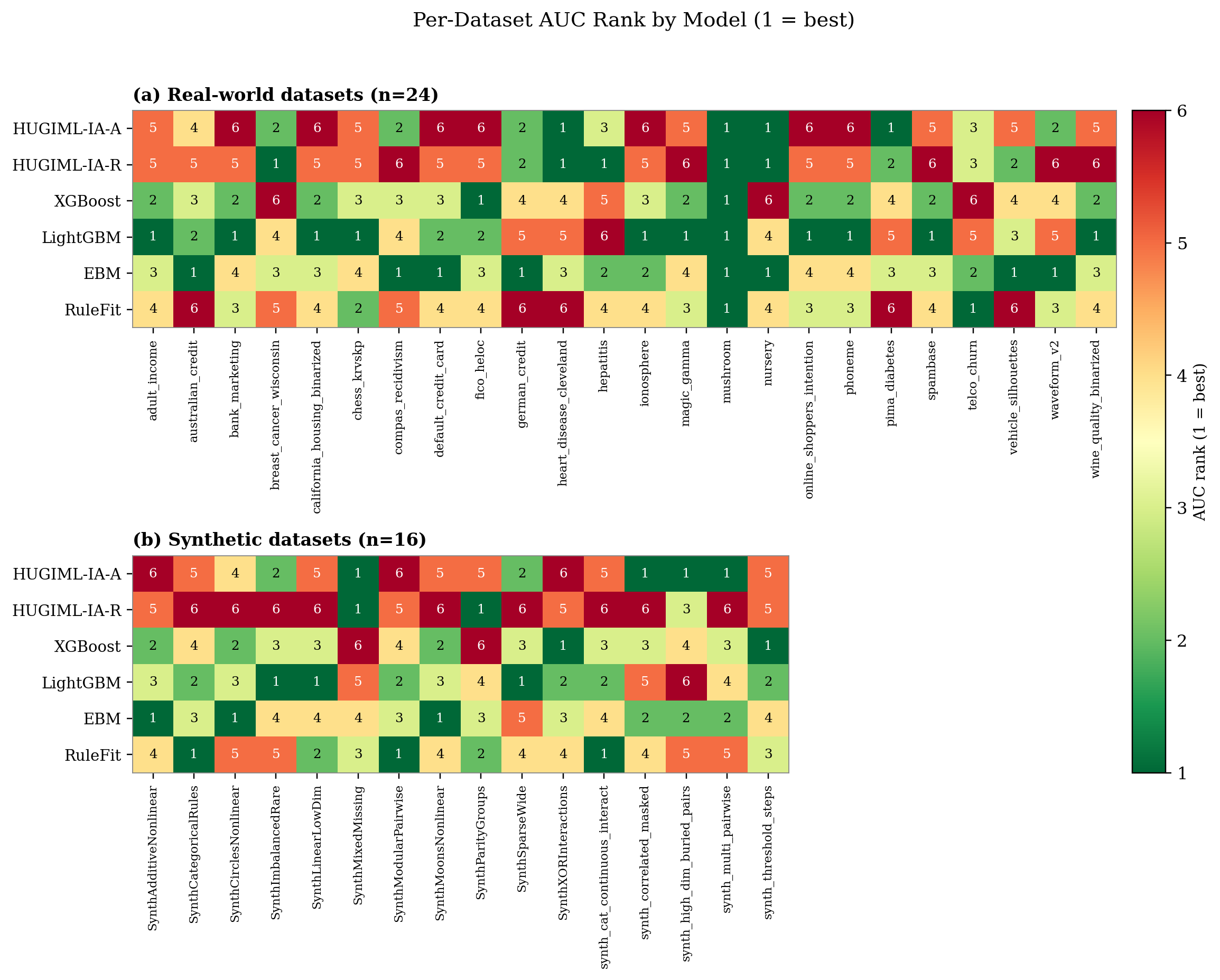}
\caption{Per-dataset AUC rank (1 = best) for all six models across the 40-dataset panel: (a) real-world datasets (n=24); (b) synthetic datasets (n=16). No model wins outright on every dataset in either group, and rank volatility is visibly higher on the real-world subset than on the synthetic subset.}
\label{fig:rank_heatmap}
\end{figure}

\begin{table}[htbp]
\centering
\caption{Number of datasets (out of 40) on which each model attains each AUC rank (1 = best). IA-A/IA-R = IAIML-A/IAIML-R.}
\label{tab:rank_distribution}
\footnotesize
\begin{tabular}{lrrrrrr}
\toprule
\textbf{Model} & \textbf{Rank 1} & \textbf{Rank 2} & \textbf{Rank 3} & \textbf{Rank 4} & \textbf{Rank 5} & \textbf{Rank 6} \\
\midrule
IA-A & 8 & 6 & 2 & 2 & 12 & 10 \\
IA-R & 7 & 3 & 2 & 0 & 14 & 14 \\
XGBoost & 4 & 11 & 11 & 8 & 1 & 5 \\
LightGBM & 14 & 8 & 4 & 5 & 7 & 2 \\
EBM & 11 & 6 & 12 & 10 & 1 & 0 \\
RuleFit & 5 & 3 & 7 & 14 & 6 & 5 \\
\bottomrule
\end{tabular}
\end{table}

LightGBM attains the most outright wins (14 of 40 datasets), followed by EBM (11) and \IAA{} (8); EBM concentrates most of its mass in ranks 1--3 (29 of 40, 72.5\%), consistent with its leading mean rank in Section~\ref{sec:main_results}. RuleFit's distribution is centered on rank 4 (14 of 40 datasets), reflecting consistent middle-of-the-pack performance. The two IAIML variants show different shapes: \IAR{} is sharply bimodal, combining 7 outright wins with a heavy concentration in the bottom two ranks (28 of 40, 70\%), while \IAA{} is somewhat more spread out, winning on 8 datasets but still finishing fifth or sixth on 22 of 40 (55\%). No single model wins on a majority of datasets, LightGBM's 14 of 40 is the largest share, reinforcing the observation in Section~\ref{sec:robustness} that dataset-to-dataset variation dominates the differences between models.

\subsection{Hyperparameter Selection Frequency}
\label{app:hyperfreq}

Table~\ref{tab:hyperfreq} reports how often each hyperparameter value was selected by the inner-CV loop, aggregated over the $40 \times 5 = 200$ outer-training-fold selections underlying Table~\ref{tab:main_validation}. The tree ensembles lean toward higher-capacity configurations (more estimators, deeper trees or more leaves), EBM leans toward finer binning and a higher learning rate, and RuleFit leans toward more rules and larger trees. The two IAIML variants diverge slightly from each other: \IAA{} is roughly balanced between pattern lengths $L=1$ and $L=2$ with a lean toward the stricter information-gain threshold ($G=0.001$), while \IAR{} is roughly balanced on both pattern length and threshold, consistent with relaxed mining's reliance on breadth of admitted candidates rather than a small set of high-information patterns. Mean fit and tuning times per model are reported in Table~\ref{tab:runtime_main} (Section~\ref{sec:comp_cost}).

\begin{table}[htbp]
\centering
\caption{Nested-CV inner-loop hyperparameter selection frequency, out of 200 outer-training-fold selections per model (40 datasets $\times$ 5 outer folds).}
\label{tab:hyperfreq}
\small
\begin{tabular}{lll}
\toprule
\textbf{Model} & \textbf{Parameter} & \textbf{Selection frequency} \\
\midrule
\IAA{} & $L$ & $2$: 57\%, \; $1$: 42\% \\
 & top-$K$ & $100$: 57\%, \; $200$: 42\% \\
 & $G$ & $0.001$: 50\%, \; $0.01$: 50\% \\
\IAR{} & $L$ & $2$: 54\%, \; $1$: 46\% \\
 & top-$K$ & $100$: 54\%, \; $200$: 46\% \\
 & $G$ & $0.01$: 62\%, \; $0.001$: 38\% \\
XGBoost & $n_{\mathrm{estimators}}$ & $200$: 68\%, \; $100$: 32\% \\
 & max depth & $4$: 61\%, \; $3$: 39\% \\
 & learning rate & $0.1$: 68\%, \; $0.03$: 32\% \\
LightGBM & $n_{\mathrm{estimators}}$ & $200$: 60\%, \; $100$: 40\% \\
 & num.\ leaves & $15$: 55\%, \; $31$: 46\% \\
 & learning rate & $0.1$: 58\%, \; $0.03$: 42\% \\
EBM & max bins & $64$: 57\%, \; $32$: 42\% \\
 & learning rate & $0.05$: 68\%, \; $0.01$: 32\% \\
RuleFit & max rules & $100$: 57\%, \; $200$: 42\% \\
 & tree size & $10$: 67\%, \; $5$: 33\% \\
\bottomrule
\end{tabular}
\end{table}

\newpage 
\section{Hyperparameter Configuration}
\label{app:hyperparams}

All models are tuned under a fully nested cross-validation protocol: an outer stratified 5-fold split provides train/test partitions for performance estimation, and within each outer training fold an inner stratified 3-fold loop selects hyperparameters from the grids below. The selected configuration is refit on the full outer training fold before scoring on the held-out outer test fold.

\begin{table}[htbp]
\centering
\caption{IAIML component grids used for nested model selection. Both pathways use $B=-1$ with adaptive binning, top-$K\in\{100,200\}$, $G\in\{0.01,0.001\}$, and original-plus-pattern mode.}
\label{tab:hyper_iaiml_validation}
\small
\begin{tabular}{llll}
\toprule
\textbf{Pathway} & \textbf{Pair operators} & \textbf{Relaxed mining} & \textbf{Additional values} \\
\midrule
\IAA{} & On & Off & $L\in\{1,2\}$; interaction-information mode \\
\IAR{} & Off & On & $L\in\{1,2\}$ \\
\bottomrule
\end{tabular}
\end{table}

\begin{table}[htbp]
\centering
\caption{Baseline model grids used for nested model selection.}
\label{tab:hyper_baseline_validation}
\small
\begin{tabular}{lll}
\toprule
\textbf{Model} & \textbf{Parameter} & \textbf{Values} \\
\midrule
XGBoost & $n_{\mathrm{estimators}}$ & $\{100,200\}$ \\
XGBoost & max depth & $\{3,4\}$ \\
XGBoost & learning rate & $\{0.03,0.1\}$ \\
\midrule
LightGBM & $n_{\mathrm{estimators}}$ & $\{100,200\}$ \\
LightGBM & learning rate & $\{0.03,0.1\}$ \\
LightGBM & num leaves & $\{15,31\}$ \\
\midrule
EBM & interactions & $\{5\}$ \\
EBM & max bins & $\{32,64\}$ \\
EBM & learning rate & $\{0.01,0.05\}$ \\
EBM & max rounds & $\{200\}$ \\
\midrule
RuleFit & $n_{\mathrm{estimators}}$ & $\{100\}$ \\
RuleFit & tree size & $\{5,10\}$ \\
RuleFit & max rules & $\{100,200\}$ \\
\bottomrule
\end{tabular}
\end{table}

\newpage 
\section{Full Ablation Results}
\label{app:ablation_full}

Table~\ref{tab:ablation_full} reports all 19 component-isolated ablation configurations evaluated on the 40-dataset panel described in Section~\ref{sec:ablation}, sorted by mean AUC. Configuration names encode the enabled mechanism (\texttt{Fixed}/\texttt{Adaptive} binning, \texttt{Plain}/\texttt{RelaxedMining}/\texttt{AugPair} interaction handling, \texttt{PatternsOnly} feature mode, pattern length \texttt{L1}/\texttt{L2}/\texttt{L3}, and source/survivor-set size or top-$K$ budget where applicable).

\begin{table}[htbp]
\centering
\caption{Full component-isolated ablation results, 40-dataset panel, sorted by mean AUC.}
\label{tab:ablation_full}
\footnotesize
{%
\small
\begin{tabular}{lrr}
\toprule
\textbf{Configuration} & \textbf{Mean AUC} & \textbf{Mean Complexity} \\
\midrule
RelaxedMining\_L2 & 0.9074 & 124.9 \\
RelaxedMining\_L2\_G0001\_Survivors20 & 0.9070 & 126.8 \\
RelaxedMining\_L2\_Survivors20 & 0.9068 & 124.7 \\
RelaxedMining\_L2\_Survivors30 & 0.9067 & 124.7 \\
RelaxedMining\_L2\_TopK200\_Survivors20 & 0.9061 & 217.6 \\
AugPair\_L2\_TopK200\_Source20 & 0.9027 & 235.1 \\
AugPair\_L2\_Source20 & 0.9023 & 153.4 \\
AugPair\_L2\_Source30 & 0.9021 & 153.4 \\
AugPair\_L2 & 0.9017 & 153.4 \\
AugPair\_L2\_StrictTopK200\_Source20 & 0.9012 & 175.1 \\
PatternsOnly\_Relaxed\_L2 & 0.8973 & 93.5 \\
AugPair\_L2\_MarginalIG\_Source20 & 0.8529 & 115.8 \\
Adaptive\_G0001\_L1 & 0.8500 & 67.3 \\
RelaxedMining\_L3\_Survivors20 & 0.8493 & 113.7 \\
Plain\_L2 & 0.8479 & 73.5 \\
Fixed\_B7\_L1 & 0.8458 & 49.0 \\
Adaptive\_topK200\_L1 & 0.8452 & 52.8 \\
Adaptive\_Bauto\_L1 & 0.8449 & 49.5 \\
PatternsOnly\_Plain\_L2 & 0.8211 & 42.1 \\
\bottomrule
\end{tabular}%
}
\end{table}

\end{appendix}
\end{document}